\documentclass[twocolumn,journal]{IEEEtran}

\usepackage{tabularx}
\usepackage{array}

\usepackage{amsmath,amsfonts,bm,bbm,enumitem}
\allowdisplaybreaks
\usepackage[mathscr]{euscript}
\usepackage{color}
\usepackage{amsthm}
\usepackage{graphicx}
\usepackage{url}
\usepackage{algorithm,algorithmic,multirow,hhline}
\usepackage{dblfloatfix}
\usepackage{amssymb}
\usepackage{mathalfa} 
\usepackage{caption}
\usepackage{comment}
\usepackage{subcaption}

\usepackage{pstricks}
\usepackage[countmax]{subfloat}
\usepackage{setspace}
\usepackage{hyperref}
\usepackage{color}
\usepackage{booktabs,multirow}
\usepackage{enumitem}
\usepackage{mathtools}
\usepackage{amssymb}
\usepackage{tikz}
\usepackage{booktabs,multirow}
\usetikzlibrary{patterns,arrows}

\usepackage{cite}

\makeatletter

\newcommand{\Rmnum}[1]{\expandafter\@slowromancap\romannumeral #1@}
\makeatother

\newtheorem{theorem}{Theorem}
\newtheorem{lemma}{Lemma}

\DeclareMathOperator*{\minimize}{minimize}
\newcommand{\undertriangleleft}{\underline{\triangleleft}}

\newtheorem{assumption}{Assumption}
\providecommand{\propositionname}{Proposition}
\usepackage{eqparbox}
\usepackage[T1]{fontenc}
\usepackage{etoolbox}

\begin{document}

% To be changed
\title{Sheaf-Based Decentralized Multimodal Learning for Next-Generation Wireless Communication Systems}
\author{Abdulmomen~Ghalkha, Zhuojun~Tian, Chaouki~Ben~Issaid, and Mehdi~Bennis
\thanks{
A.~Ghalkha, Z.~Tian, C.~B.~Issaid, and M.~Bennis are with the Center for Wireless Communications, University of Oulu, Oulu 90014, Finland. Email: \{abdulmomen.ghalkha, zhuojun.tian, chaouki.benissaid, mehdi.bennis\}@oulu.fi.
}}

\maketitle

\begin{abstract}
In large-scale communication systems, increasingly complex scenarios require more intelligent collaboration among edge devices collecting various multimodal sensory data to achieve a more comprehensive understanding of the environment and improve decision-making accuracy. However, conventional federated learning (FL) algorithms typically consider unimodal datasets, require identical model architectures, and fail to leverage the rich information embedded in multimodal data, limiting their applicability to real-world scenarios with diverse modalities and varying client capabilities. To address this issue, we propose Sheaf-DMFL, a novel decentralized multimodal learning framework leveraging sheaf theory to enhance collaboration among devices with diverse modalities. Specifically, each client has a set of local feature encoders for its different modalities, whose outputs are concatenated before passing through a task-specific layer. While encoders for the same modality are trained collaboratively across clients, we capture the intrinsic correlations among clients' task-specific layers using a sheaf-based structure. To further enhance learning capability, we propose an enhanced algorithm named Sheaf-DMFL-Att, which tailors the attention mechanism within each client to capture correlations among different modalities. A rigorous convergence analysis of Sheaf-DMFL-Att is provided, establishing its theoretical guarantees. Extensive simulations are conducted on real-world link blockage prediction and mmWave beamforming scenarios, demonstrate the superiority of the proposed algorithms in such heterogeneous wireless communication systems.
\end{abstract}
\begin{IEEEkeywords}
Distributed optimization, decentralized learning, multimodal federated learning,  next-generation wireless communication.
\end{IEEEkeywords}

%\vspace{-10pt}

\section{Introduction}
Machine learning (ML) has emerged as a key enabler for next-generation wireless communication systems \cite{saad2019vision}, tackling core challenges such as blockage prediction \cite{wu2023proactively, nishio2019proactive}, millimeter wave (mmWave) beamforming \cite{li2020millimeter, chen2020hybrid, charan2022towards}, and channel estimation \cite{yang2020deep}. While these methods offer promising performance, their success heavily depends on access to large, high-quality datasets. In large-scale wireless networks, however, data is often distributed across numerous devices, each with limited and non-identically distributed samples. Federated learning (FL) has been introduced as a promising alternative, enabling multiple clients to collaboratively train ML models without sharing their raw data \cite{ghalkha2024din, ghalkha2024scalable}. By preserving data privacy and reducing the need for centralized data aggregation, FL offers an attractive solution for collaborative intelligence in wireless communication problems \cite{dai2022distributed, niknam2020federated, ding2023distributed}.

\begin{figure}[t]
\centering
\includegraphics[width=1.0\linewidth]{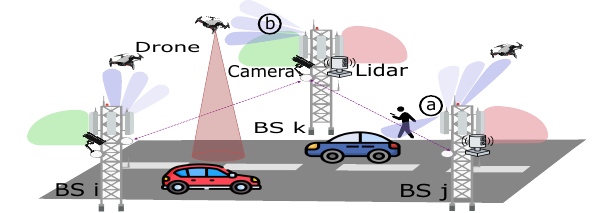}
\caption{Illustration of multimodal use cases in large-scale wireless communication systems: (a) Multimodal-aided blockage prediction, (b) Multimodal beam prediction.}
\label{fig:system_model}
\vspace{-0.5cm}
\end{figure}

Despite notable advancements in FL, the majority of existing research is constrained to unimodal learning settings, where all clients operate on the same data modality, while employing identical model architectures \cite{song2024tackling}. However, this assumption fails to reflect the complexity and heterogeneity of real-world wireless communication systems, where clients observe data through diverse modalities, including visual inputs, acoustic data, or RF signals.
Figure~\ref{fig:system_model} provides a concrete motivating example illustrating two typical multimodal wireless communication scenarios. In scenario (a), a multimodal-aided blockage prediction setting is shown, where various sensing inputs are used to predict potential link disruptions/blockages. Scenario (b) depicts a beam prediction task, where the device utilizes multimodal data such as LiDAR and images to predict the optimal beam direction toward a moving drone. 
The benefits of multimodal learning in communication systems have also been demonstrated in prior works \cite{mollah2025multi, wang2024wireless, yang2020deep}. Specifically, results in \cite{yang2020deep} validate that combining received signal strength, user location, and uplink channel state information significantly improves downlink channel prediction accuracy in massive MIMO systems. 
These findings underscore the need for distributed learning frameworks that can effectively harness multimodal information. Motivated by this, we aim to design a distributed multimodal learning framework that leverages diverse modalities in real-world communication settings.

Several recent works have explored the challenges and opportunities of multimodal federated learning, particularly in the presence of modal heterogeneity among clients. For example, \cite{yu2023multimodal} introduced contrastive objectives to align representations across and within modalities using a public dataset, while \cite{ouyang2023harmony} proposed Harmony, which disentangles the training into a two-stage process consisting of modality-wise federated learning followed by federated fusion. The authors in \cite{chen2022towards} proposed simultaneously optimizing both the blending of modalities and the aggregation weights of local models by adaptively measuring their overfitting and generalization behaviors. Furthermore, \cite{chen2022fedmsplit} introduced a dynamic multi-view graph to model and exploit relationships among clients for improved multimodal learning, specifically by adaptively capturing correlations among multimodal models to address modality discrepancy. Although these methods validate the benefits of incorporating multiple modalities, they typically rely on a central parameter server (PS) to coordinate training and often assume access to public datasets for cross-modal alignment. This centralized approach introduces a single point of failure, limits scalability, and may be infeasible in distributed environments like Internet of Things (IoT) networks, where devices can only communicate via peer-to-peer links.

To address these limitations, \cite{yu2023multimodal} proposed DMML-KD, a fully decentralized multimodal FL framework, which uses knowledge distillation to facilitate knowledge sharing across heterogeneous modalities. While this approach mitigates some of the coordination and communication challenges, it suffers from instability in encoder alignment, as clients sharing the same modality may still learn divergent encoders due to the local knowledge distillation process and varying modality contexts. Consequently, there remains a pressing need for a fully decentralized multimodal learning framework that can (i) accommodate diverse client architectures and data modalities, and (ii) facilitate efficient information sharing and collaboration without reliance on a PS. Recent work by Ben Issaid et al. \cite{ben2024tackling} has explored the application of sheaf theory in decentralized multi-task learning, demonstrating its effectiveness in modeling feature and sample heterogeneity across clients. Sheaf structures \cite{robinson2014topological}, originating from algebraic topology, offer a mathematically principled framework for rigorously integrating local data or models across a network. However, this work does not consider the unique challenges posed by multimodal FL, where clients may have access to different combinations of data modalities.

In this work, we extend the sheaf-theoretic approach to the decentralized multimodal federated learning (DMFL) and demonstrate its applicability across diverse 6G scenarios. Built upon the decentralized multi-task learning paradigm introduced in \cite{ben2024tackling}, we propose Sheaf-DMFL, treating each unique modality combination as a separate task, and the sheaf structure is leveraged to enable nuanced collaboration across clients with heterogeneous sensors by explicitly modeling and learning \textit{consistency relationships} between these differing tasks.
Furthermore, by incorporating a local attention mechanism in the embeddings' fusion, Sheaf-DMFL-Att is proposed as an enhanced variant, which improves multimodal representation learning and shows better predictive performance in complex scenarios. Our contributions can be summarized as follows:
\begin{itemize}
\item We propose a partially-shared model architecture where clients collaboratively train common encoders for shared modalities. The modality-specific features are then fused and passed through a client-specific \emph{head} layer, framing the problem as federated multi-task learning (FMTL) where inter-task relationships are learned using the sheaf theoretic principles.
\item We propose Sheaf-DMFL, incorporating a sheaf Laplacian to capture and learn the inter-task relationships across clients with varying modality availability, which is jointly optimized during training, facilitating adaptive and efficient collaboration among clients with diverse modalities in a principled and scalable manner.
\item Furthermore, we introduce the local attention mechanism for multimodal fusion, named Sheaf-DMFL-Att, which demonstrates improved performance in complex scenarios. We also provide a detailed convergence analysis for Sheaf-DMFL-Att.
\item The algorithms are simulated in two real-world multi-modal 6G communication tasks, including a multi-base-station blockage prediction and mmWave beamforming prediction tasks. Experimental results demonstrate that both algorithms achieve superior learning performance and fast convergence in the presence of heterogeneous modalities and decentralized connectivity. 
\end{itemize}

The rest of this paper is organized as follows. Section \ref{formulation} presents the system model and problem formulation. Section \ref{proposed_algorithm} introduces the proposed algorithms. Section \ref{convergence_analysis} provides the theoretical convergence analysis. Section \ref{Simulation_section} evaluates the proposed methods through real-world 6G use cases, including blockage prediction and mmWave beamforming. Finally, Section \ref{sec_conclusion} concludes the paper and outlines future directions. Note that this article significantly extends our previous work \cite{Ghalkha2025SheafDMFL} in several ways. 
Firstly, we extend the concatenation fusion to the attention-based fusion, enhancing the performance in more complex scenarios. Secondly, the theoretical analysis of the convergence property is rigorously derived. Finally, we conduct additional experiments on an additional 6G use case.

\section{System Model and Problem Formulation} 
\label{formulation}
\subsection{Decentralized Multimodal Multi-task Learning}
We consider a decentralized communication network denoted by $\mathcal{G} = (\mathcal{V}, \mathcal{E})$, where $\mathcal{V} = [N] = \{1, \ldots, N\}$ represents the set of clients and $\mathcal{E}\subseteq\mathcal{V}\times\mathcal{V}$ denotes the edges indicating direct communication links between clients. The set of global modalities is defined as $\mathcal{M}=\{1, \dots, M\}$, where $M$ is the total number of distinct modalities in the system. 
For an arbitrary client $i$, its local training dataset is denoted by $\mathcal{D}_i=\{\bm{x}_{i}, \bm{y}_{i}\}$, with $\bm{x}_{i}=\{\bm{x}_{i,k}: \forall k\in\mathcal{M}_i\}$. Here, $\mathcal{M}_i$ represents the locally available modalities of node $i$, where $ \mathcal{M}_i\subseteq \mathcal{M}$ and $|\mathcal{M}_i| \geq 1$. 
The local model parameters of node $i$ are denoted by $\bm{\theta}_i \in \mathbb{R}^{d_i}$, where $d_i$ is the number of parameters. 
This model, denoted by $\chi_{i}: \prod_{k \in \mathcal{M}_i} \mathbb{R}^{m_k} \to \mathbb{R}^C$, takes the raw data inputs from the available modalities and produces a prediction in the label space, where $m_k$ represents the input dimension of modality $k$ and $C$ is the dimension of the output space.

In the FL setup, multiple clients train models collaboratively while preserving data privacy by not sharing local data. Depending on the modalities available to a client, a task is defined by a specific combination of input modalities, and clients solving the same task must use the same set of modalities. However, clients solving different tasks may still share common modalities, creating related tasks. Traditional FL focuses on identical tasks and often ignores clients working on different tasks, even when they share modalities. Yet, leveraging these shared modalities can enhance learning. This modality-based cross-task relationship is captured by a task similarity matrix, $\bm{\Omega}$, which enables clients solving related tasks to transfer knowledge, enhancing learning efficiency beyond classic FL approaches \cite{smith2017federated}. A standard approach to encourage collaboration in multi-task learning is to introduce a regularization term that promotes similarity between the models of related tasks \cite{smith2017federated}. Let $\bm{\theta} = [\bm{\theta}_1^T, \ldots, \bm{\theta}_N^T]^T \in \mathbb{R}^d$ denote the stacked parameters for the entire system, where $d = \sum_{i=1}^N d_i$ is the sum of the number of model parameters of all clients. Then, the optimization problem of multi-task learning can be written as \cite{smith2017federated}
\begin{align}
\label{eq:MTL_problem}
\minimize_{\bm{\theta}, \bm{\Omega}} \sum_{i=1}^{N} f_i(\bm{\theta}_i) + \mathcal{R}(\bm{\Omega}, \bm{\theta}), \tag{P1}
\end{align}
where $f_i(\bm{\theta}_i) = \sum_{j=1}^{N_i}\ell(\hat{y}_{i, j}, y_{i,j};\bm{\theta}_i)$ is the loss function of the $i$th client on the $j$th data sample, depending on the model parameters $\bm{\theta}_i$. Here, $\hat{y}_{i, j}$ is the predicted label for the $j$th data sample, and $\ell$ denotes the loss between the predicted label and the ground truth $y_{i,j}$. The regularizer $\mathcal{R}(\bm{\Omega}, \bm{\theta})$ controls the relationships between tasks, encoded in the matrix $\bm{\Omega}$. 

This formulation enables distributed optimization, where each client uses local data to solve its task while considering task dependencies encoded in $\bm{\Omega}$. 

The authors in \cite{smith2017federated} propose solving (\ref{eq:MTL_problem}) by alternating between optimizing the model parameters $\bm{\theta}$ and the task relationships $\bm{\Omega}$. While this approach simplifies the optimization process, the update of $\bm{\Omega}$ requires a central processor, which introduces significant communication bottlenecks and may not be feasible in IoT networks.
To address this issue, we leverage the sheaf structure to model and learn such task dependencies.

\subsection{Sheaf Structure for Multi-task Learning}
The sheaf structure, as defined in \cite{ben2024tackling}, enables modeling interactions between tasks and finding pairwise model relations. We begin by establishing the mathematical foundation for defining sheaves over communication graphs and then proceed to propose our method.
Intuitively, a sheaf structure can be seen as a mathematical framework connecting local views (represented by vector spaces) over the network.
In distributed learning, each client has its own perspective on the learning problem based on its available modalities. The sheaf structure coherently integrates these perspectives while preserving the underlying network topology. Specifically, the sheaf structure formalizes how these views relate to each other through restriction maps, serving as ``translation functions'' between different clients' perspectives.

\begin{figure}[bt]
    \centering
    \includegraphics[width=0.5\textwidth]{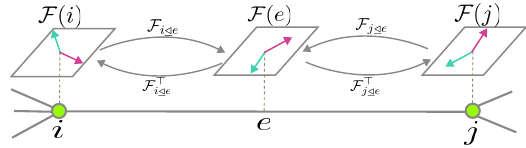}
    \caption{Illustration of a cellular sheaf over a communication graph. Vector spaces are assigned to nodes and edges, with linear restriction maps encoding the relationships between the vector spaces of adjacent nodes.}
    \label{fig:sheaf_diagram}
    \vspace{-0.5cm}
\end{figure}

A \textit{cellular sheaf}, as depicted in Figure~\ref{fig:sheaf_diagram}, equips the graph with extra structure by assigning vector spaces and linear transformations to its vertices and edges. Specifically, a ``cellular sheaf'' $\mathcal{F}$ of $\mathbb{R}$-vector spaces over a simple graph $\mathcal{G} = (\mathcal{V}, \mathcal{E})$ includes
\begin{itemize}
  \item For each $i \in \mathcal{V}$, a vector space $\mathcal{F}(i) = \mathbb{R}^{d_i}$.
  \item For each edge $e = (i, j) \in \mathcal{E}$, a vector space $\mathcal{F}(e) = \mathbb{R}^{d_{ij}}$.
  \item For each edge $e \in \mathcal{E}$ and incident vertex $i \in \mathcal{V}$, a linear transformation $\mathcal{F}_{i \undertriangleleft e} : \mathcal{F}(i) \to \mathcal{F}(e)$.
\end{itemize}
We denote $\mathcal{F}(i)$ and $\mathcal{F}(e)$ as stalks over $i$ and $e$, respectively, with $\mathcal{F}_{i \undertriangleleft e}$ being the restriction map from $i$ to $e$. The matrix representation of $\mathcal{F}_{i \undertriangleleft e}$, denoted $\bm{P}_{ij}$, is used interchangeably with $\mathcal{F}_{i \undertriangleleft e}$, while the dual maps $\mathcal{F}^*_{i \undertriangleleft e}$ correspond to the transposes $\bm{P}_{ij}^T$. For each client $i \in \mathcal{V}$, $\mathcal{F}(i)$ represents the space for their local models, $\bm{\theta}_i \in \mathcal{F}(i)$. The set of local models ${\bm{\theta}_i},{i \in \mathcal{V}}$ belongs to the total space $C^0(\mathcal{F}) := \bigoplus_{i \in \mathcal{V}} \mathcal{F}(i)$. Due to feature heterogeneity or different tasks, clients may have different model sizes, and a single client cannot fully observe elements of $C^0(\mathcal{F})$. 

Models can be compared via restriction maps if they share an edge. For $e = (i, j) \in \mathcal{E}$, $\mathcal{F}(e)$ allows comparison of models $\bm{\theta}_i$ and $\bm{\theta}_j$ using projections $\mathcal{F}_{i \undertriangleleft e}(\bm{\theta}_i)$ and $\mathcal{F}_{j \undertriangleleft e}(\bm{\theta}_j)$. The overall comparison is done in the total space $C^1(\mathcal{F}) := \bigoplus_{e \in \mathcal{E}} \mathcal{F}(e)$, using the sheaf Laplacian $L_{\mathcal{F}}$, a block matrix whose $(j, i)$-th block, denoted by $L_{j i}$, is defined as
\begin{align}
    L_{ji} &= \begin{cases} 
    \sum_{i \undertriangleleft e} \mathcal{F}^*_{i \undertriangleleft e} \circ \mathcal{F}_{i \undertriangleleft e}, & \text{if } i = j, \\
    -\mathcal{F}^*_{j \undertriangleleft e} \circ \mathcal{F}_{i \undertriangleleft e}, & \text{if } e = (i, j) \in \mathcal{E}, \\
    \bm{0}, & \text{otherwise}.
    \end{cases}
\end{align}
The sheaf Laplacian $ \bm{\theta}^T L_{\mathcal{F}} \bm{\theta} $ quantifies how much $ \bm{\theta} $ deviates from the global consensus, i.e., when $ L_{\mathcal{F}} \bm{\theta} = 0 $. This can be expressed as 
\begin{align}
\label{eq_quadratic_sheaf_laplacian}
     \bm{\theta}^T L_{\mathcal{F}} \bm{\theta}  = \sum_{e = (i,j) \in \mathcal{E}} \| \mathcal{F}_{i \undertriangleleft e}(\bm{\theta}_i) - \mathcal{F}_{j \undertriangleleft e}(\bm{\theta}_j) \|^2. 
\end{align}
Under this mathematical foundation of the sheaf structure, problem \eqref{eq:MTL_problem} can be rewritten as
\begin{align}
\label{eq_FMTL_problem2}
\minimize_{\{\bm{\theta}_i\}_{i\in\mathcal{V}}, {L_\mathcal{F}}} \sum_{i = 1}^{N} f_i(\bm{\theta}_i) + \frac{\lambda}{2}\bm{\theta}^\top L_{\mathcal{F}}\bm{\theta},
\end{align}
where $\lambda$ is a regularization hyperparameter that controls the strength of the sheaf Laplacian regularization.

\section{Proposed Learning Framework}\label{proposed_algorithm}
\begin{figure*}[!htp]
    \centering
    \begin{subfigure}[t]{0.30\textwidth}
        \centering
        \includegraphics[width=\linewidth]{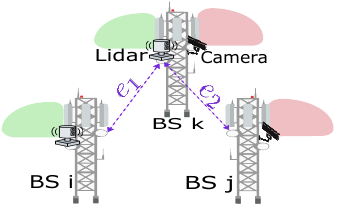}
        \caption{}
        \label{fig:subfig1}
    \end{subfigure}
    \begin{subfigure}[t]{0.68\textwidth}
        \centering
        \includegraphics[width=\linewidth]{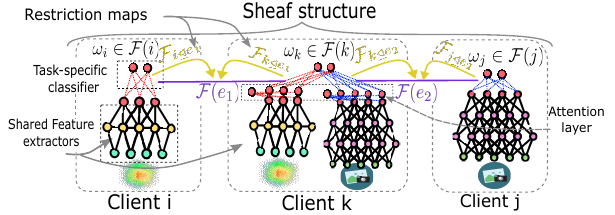} % <-- Replace with actual path
        \caption{}
        \label{fig:architecture}
    \end{subfigure}
    \caption{An illustration of three multi-modal BSs and their model architectures: The proposed sheaf-enabled decentralized multimodal federated learning architecture with three clients and two modalities. Each client employs modality-specific feature extractors and attention-based fusion, followed by a task-specific classifier. A sheaf structure is defined over the task-specific layers to model inter-client (task) relations.}
    \label{fig:system_model_architecture_combined}
    \vspace{-0.5cm}
\end{figure*}

To solve problem \eqref{eq_FMTL_problem2}, a full set of $L_{ji}$ across all edges needs to be learned, leading to significant memory and computation overhead, with space complexity scaling as $\mathcal{O}(|\mathcal{E}|d^2)$. This issue is exacerbated in multimodal settings, where multiple feature extractors lead to large models, making the formulation impractical for most machine learning architectures. To address this issue, we first introduce a partially-shared model architecture, where the sheaf Laplacian is only applied for the task-specific layer, retaining expressiveness while improving scalability.
Based on that, we further design the updating rules.

\subsection{Partially-shared Model Architectures}
As shown in \cite{collins2021exploiting, tian2023distributed}, local deep models can be split into globally shared layers and client-specific heads.
This idea is rooted in traditional ML, where heterogeneous data can share a common representation despite differing tasks. 
Inspired by this, we propose using a partially shared neural network architecture as the local model in each node. 
Specifically, as shown in Figure \ref{fig:architecture}, the initial layers, mapping the raw data input to low-dimensional representations, are shared among all clients with the same modality, while the last layer is task-specific, fusing embeddings from different modalities. In this approach, each client utilizes a set of local feature extractors, one for each modality, denoted as $\psi_{i, k}: \mathbb{R}^{m_k} \rightarrow \mathbb{R}^{l_k}$ for each modality $k \in \mathcal{M}_i$ parameterized by $\bm{\phi}_{i, k}$, where $l_k$ is the dimension of the extracted feature embedding. To form a unified representation, the extracted embeddings are concatenated into a single vector

While simple concatenation is a viable fusion strategy, it treats all modalities equally. To allow the model to dynamically prioritize more informative modalities for a given task, we propose an \textit{attention-based fusion mechanism} to compute a weighted combination of modality-specific embeddings, allowing the model to learn the interaction between modalities. This design exploits correlations between different modalities, letting the model focus on the most relevant features from each, improving robustness and overall performance. The attention weights $\alpha_{i,k}$ are computed using a learnable scoring function. One common choice is the multi-layer perceptron (MLP) attention mechanism, which computes $e_{i,k} = \tanh\left( \bm{\beta}_{i, k}^\top \bm{h}_{i,k} \right)$, where $\bm{\beta} \in \mathbb{R}^{l_k \times 1}$ is a learnable parameter \cite{vaswani2017attention}. The scores $e_{i,k}$ are then normalized using a softmax operation to get $\alpha_{i,k} = \exp(e_{i,k})/\sum_{j \in \mathcal{M}_i} \exp(e_{i,j})$.

The mapping from raw data to the concatenated/combined feature extractors can be written as a function $\psi_i$ that processes the input modalities available to the client $i$. Each data sample corresponding to modality $k \in \mathcal{M}_i$ is first encoded using a modality-specific feature extractor $\psi_k$, producing an intermediate representation $\psi_k(\bm{x}_{i,k})$. 
These modality-specific embeddings are then concatenated or fused using the attention layer to form the final feature representation for the clients, expressed as 
\begin{itemize}
    \item{Concatenation: 
    \begin{align}
    \label{eq_concatinate}
    \psi_i(\bm{x}_i) = \text{Concatenate} \left( \{ \psi_k(\bm{x}_{i,k}) \}_{k \in \mathcal{M}_i} \right),
    \end{align}}
    \item{Attention fusion:
        {
    \begin{align}
    \label{eq_concatinate_attention}
    \psi_i(\bm{x}_i) =  \sum_{k \in \mathcal{M}_i}\alpha_{i,k} \psi_k(\bm{x}_{i,k}),
    \end{align}
    }
    }
\end{itemize}
In \eqref{eq_concatinate}, $\psi_i$ is parameterized by $\bm{\Phi}_i = \{\bm{\phi}_{i, k}\}_{k \in \mathcal{M}_i}$, while in \eqref{eq_concatinate_attention} $\psi_i$ is additionally parameterized by the attention parameters $\bm{\beta}_{i} = \{\bm{\beta}_{i, k}\}_{k \in \mathcal{M}_i}$.

Once the fused feature representation is obtained, it is processed by a task-specific layer, modeled as a single-layer NN $g_i(.)$ parameterized by $\bm{\omega}_i \in \mathbb{R}^{d_i^g}$. The output of the task-specific layer is denoted by $\hat{y}_i = g_i(\psi_i(\bm{x}_i); \bm{\omega}_i)$. 
The end-to-end mapping from the raw data to the output can be given by $\chi_i = g_i \circ \psi_i (\{\bm{x}_{i,k} \}_{k \in \mathcal{M}_i})$, where $g_i $ is parameterized by $ \bm{\omega}_i $. 

Thus, for simple concatenation as in \eqref{eq_concatinate}, the overall model parameters for client $i$ is defined as $\bm{\theta}_i = [\bm{\Phi}_i^T, \bm{\omega}_i^T]^T$, consisting of the shared feature extractors and the task-specific layer. For attention fusion \eqref{eq_concatinate_attention}, the the overall model parameters for client $i$ are defined as $\bm{\theta}_i = [\bm{\Phi}_i^\top, \bm{\beta}_i^\top, \bm{\omega}_i^T]^\top$, with additional local attention parameters. Note that attention fusion introduces additional parameters to better handle complex modality interactions, making it more suitable for challenging tasks.

In the $(r + 1)$th iteration and assuming reliable communication, the client $i$ updates its feature extractor parameters $\bm{\phi}_{i,k}$ for each modality $k \in \mathcal{M}_i$ with Stochastic Gradient Descent (SGD) to obtain $\bm{\phi}_{i,k}^{r+\frac{1}{2}}$. For attention fusion, the clients update their local attention parameters $\bm{\beta}_{i,k}$ at the same time.
After applying local SGD, each client aggregates the intermediate parameters of the feature extractors from its neighbors, which share the same modality. To formalize this, we introduce a modality-specific mixing matrix $\bm{W}_k = [w_{ij}^k]$. The modality-specific mixing matrices $\bm{W}_k$ are constructed to ensure effective aggregation of encoder parameters among clients sharing the same modality $k$. For each modality $k$, we first identify the subgraph $\mathcal{G}_k$ consisting of clients with that modality. Then, we construct $\bm{W}_k$ using the Metropolis-Hastings weight assignment. The feature extractor parameters are then updated as
\begin{align}
    \vspace{-0.5cm}
    \label{eq_theta_aggregate}
    \bm{\phi}_{i,k}^{r+1}\!=\!\sum_{j \in \mathcal{N}_i^k} w_{ij}^k \bm{\phi}_{j,k}^{r+\frac{1}{2}},
\end{align}
where $\mathcal{N}_i^r$ denotes the set of neighbors of client $i$ that possess modality $k$. 
Concerning the task-specific parameters $\bm{\omega}_i$, we introduce the sheaf-assisted aggregation procedure.

\subsection{Sheaf-enabled Decentralized Multimodal Federated Learning}
Under the partially shared model architecture, we capture the interaction among tasks by defining the sheaf structure on the task-specific layers. 
The dimensionality of the task-specific layer is much smaller ($d_i^g \ll d$), significantly reducing memory complexity. Figure \ref{fig:architecture} illustrates the architecture of the partially-shared models, where a sheaf structure is defined on task-specific classifiers, with feature extractors shared across clients. This setup enables the projection and comparison of different task-specific parameters in the edge vector space $\mathcal{F}(e)$.
Using the sheaf Laplacian regularization, the FMTL problem (\ref{eq_FMTL_problem2}) is formulated as
\begin{align}
    \label{eq_FMTL_problem}
        \minimize_{\substack{\{\bm{\Phi}_i, \bm{\beta}_i, \bm{\omega}_i\}_{i \in \mathcal{V}}, \\ \{\bm{P}_{ij}\}_{(i, j) \in \mathcal{E}}}} \sum_{i = 1}^{N} f_i(\bm{\Phi}_i, \bm{\beta}_i, \bm{\omega}_i) + \frac{\lambda}{2} \bm{\omega}^T L_{\mathcal{F}}(\bm{P})\bm{\omega},
\end{align}
where $\bm{\omega}$ is the concatenation of $\bm{\omega}_i$. The choice of $\lambda$ determines the degree of coordination between neighboring clients' task-specific parameters. When $\lambda = 0$, clients train independently. As $\lambda$ increases, the regularization penalizes inconsistencies across projected task-specific layers more strongly, thereby encouraging tighter collaboration. The matrix $L_{\mathcal{F}}(\bm{P})$ denotes the sheaf Laplacian, which is constructed using the restriction block matrix $\bm{P}$ with blocks $P_{e, i}$ corresponding to row $e \in \mathcal{E}$ and column $i \in \mathcal{V}$, defined as
\begin{align}
\label{eq_projection_block_matrix}
    \bm{P}_{e, i} = \begin{cases}
        \bm{P}_{ij}, & \text{if} \ e=(i, j),\\
        -\bm{P}_{ij}, & \text{if} \ e=(j, i), \\
        \bm{0}, & \text{otherwise,}
    \end{cases}
\end{align}
where $\bm{P}_{ij} \in \mathbb{R}^{d_{ij} \times d_i}$ is a restriction map projecting the local model $\bm{\theta}_i$ onto a shared edge space with neighbor $j$. 

Using the definition of Sheaf Laplacian in \eqref{eq_quadratic_sheaf_laplacian} and construction of restriction maps in \eqref{eq_projection_block_matrix}, we can write the quadratic term in \eqref{eq_FMTL_problem} as a sum of pairwise disagreements over the network edges
\begin{align}
\label{eq_quadratic_sheaf_laplacian_omega}
\bm{\omega}^T L_{\mathcal{F}}(\bm{P})\bm{\omega} = \sum_{i=1}^{N}\sum_{j \in \mathcal{N}_i} \|\bm{P}_{ij}\bm{\omega}_i\!-\!\bm{P}_{ji}\bm{\omega}_j\|^{2},
\end{align}
where $\bm{P}_{ij} \in \mathbb{R}^{d_{ij}\times d_i^g}$ is the restriction map of the model $\bm{\omega}_i$ to the edge $e=(i, j)$. 
Conceptually, the restriction maps $\bm{P}_{ij}$ and $\bm{P}_{ji}$ project the parameter vectors $\bm{\omega}_i$ and $\bm{\omega}_j$ into a common, lower-dimensional ``comparison space'' defined over the edge  $(i, j)$. The sheaf Laplacian term then penalizes disagreements within this shared space. By learning the maps 
$\bm{P}$ jointly with the models $\bm{\omega}$, the system adaptively discovers the optimal way to compare and align heterogeneous task models. Using the definition in \eqref{eq_quadratic_sheaf_laplacian_omega}, we can rewrite the FMTL problem in \eqref{eq_FMTL_problem} as
\begin{align}
   \label{eq_FMTL_problem3}
    \minimize_{\substack{\{\bm{\Phi}_i, \bm{\beta}_i, \bm{\omega}_i\}_{i \in \mathcal{V}}, \\ \{\bm{P}_{ij}\}_{(i, j) \in \mathcal{E}}}} \sum_{i = 1}^{N} f_i(\bm{\Phi}_i, \bm{\beta}_i, \bm{\omega}_i)\!+\!\frac{\lambda}{2} \sum_{i=1}^{N}\sum_{j \in \mathcal{N}_i} \|\bm{P}_{ij}\bm{\omega}_i\!-\!\bm{P}_{ji}\bm{\omega}_j\|^{2}.
\end{align}
To solve (\ref{eq_FMTL_problem3}), each client $i$ jointly learns $\bm{\theta}_i$ and $\bm{P}_{ij}, \forall j \in \mathcal{N}_i$ by minimizing its local loss while keeping track of the common features of the neighboring clients by minimizing the second term. 
Specifically, in one iteration, each client $i$ exchanges $\bm{P}_{ij}^r \bm{\omega}_i^r$ and $\bm{P}_{ji}^r \bm{\omega}_j^r$ with all of its neighbors $j \in \mathcal{N}_i$, to update the task-specific parameters $\bm{\omega}_i$. The update rule for the task-specific model parameters $\bm{\omega}_i$ is as follows \cite{ben2024tackling}
\begin{align}
    \label{eq:task_specific_update}
    \bm{\omega}_{i}^{r+1}\!=\!\bm{\omega}_{i}^{r}\!-\!\alpha 
    ( 
    \nabla_{\bm{\omega}_i} f_i(\bm{\theta}_i^r)\!+\!\lambda 
    \sum_{j \in \mathcal{N}_i} 
    (\bm{P}_{ij}^{r})^T \left( \bm{P}_{ij}^{r} \bm{\omega}_{i}^{r}\!-\!\bm{P}_{ji}^{r} \bm{\omega}_{j}^{r} )
    \right),
\end{align}
where $\alpha$ is the learning rate and $\nabla f_i(\bm{\omega}_{i}^{r})$ is the gradient of the local loss function of client $i$ at iteration $k$. After updating $\bm{\omega}_i$, client $i$ computes $\bm{P}_{ij}^k \bm{\omega}_i^{r+1}$, then sends it and receives $\{\bm{P}_{ji}^k \bm{\omega}_j^{r+1}\}_{j \in \mathcal{N}_i}$ to update the projection maps $\bm{P}_{ij}$ with the following rule
\begin{align}
    \label{eq:projection_update}
    \bm{P}_{ij}^{r+1} &= \bm{P}_{ij}^{r} - \eta \lambda 
    \left( \bm{P}_{ij}^{r} \bm{\omega}_{i}^{r+1} - \bm{P}_{ji}^{r} \bm{\omega}_{j}^{r+1} \right)
    \left( \bm{\omega}_{i}^{r+1} \right)^T,
\end{align}
where $\eta$ is the learning rate for $\bm{P}$. The Sheaf-enabled Decentralized Multimodal Federated Learning (Sheaf-DMFL) and its attention-enhanced counterpart Sheaf-DMFL-Att can be summarized in Algorithm \ref{alg:fmtl_sheaf}.

\begin{algorithm}[t]
\caption{Sheaf-enabled Decentralized Multimodal Federated Learning}
\label{alg:fmtl_sheaf}
\begin{algorithmic}[1]
\STATE \textbf{Input:} Initial parameters $\bm{\Phi}_i$, $\bm{\omega}_i$, $\bm{P}_{ij}$, learning rates $\eta_{\phi}$, $\alpha$, $\eta$, mixing matrices $\bm{W}_k$
\FOR{each iteration $r = 1, 2, \dots$}
    \FOR{each client $i \in \mathcal{V}$ \textbf{in parallel}}
        \FOR{each modality $k \in \mathcal{M}_i$}
            \STATE Compute local feature embeddings $\bm{h}_{i,k} = \psi_{i, k}(\bm{x}_{i,k})$.
        \ENDFOR
        \STATE Concatenate embeddings to form a unified representation using \eqref{eq_concatinate} (Sheaf-DMFL) or \eqref{eq_concatinate_attention} (Sheaf-DMFL-Att).
        \STATE Compute task-specific output $\hat{y}_i = g_i(\hat{\bm{h}}_i; \bm{\omega}_i)$.
        \STATE Update feature extractor parameters $\bm{\Phi}_i$ and attention parameters $\bm{\beta}_i$ with SGD.
        \STATE Aggregate feature extractor parameters from neighbors using \eqref{eq_theta_aggregate}.
        \STATE Update task-specific parameters using \eqref{eq:task_specific_update}.
        \STATE Exchange projections $\bm{P}_{ij}^r \bm{\omega}_i^{r+1}$ with neighbors.
        \STATE Update projection maps using \eqref{eq:projection_update}.
    \ENDFOR
\ENDFOR
\end{algorithmic}
\end{algorithm}
The projection matrix $\mathbf{P}_{i,j}$ can be initialized either randomly or structurally. In the case of random initialization, each element of $\mathbf{P}_{i,j}^{\mathcal{R}}$ is drawn from a normal distribution, $\mathcal{N}(0, \sigma^2)$, with a spreading factor $\sigma^2$. In the case of structured initialization, $\mathbf{P}_{i,j}^{\mathcal{I}}$ is constructed from a subset of columns of the identity matrix. The dimensions of $\mathbf{P}_{i,j}$ are determined by the compression factor $\gamma$. Specifically, $\mathbf{P}_{i,j}$ has dimensions $\left( \left\lfloor \gamma \cdot \frac{\left( d_i^g + d_j^g \right)}{2} \right\rfloor, \, d_i \right)$, where $\gamma$ controls the projection dimension, and $d_i$ and $d_j$ are the number of parameters of classifiers $i$ and $j$, respectively. The compression factor $\gamma$ thus scales the number of rows in $\mathbf{P}_{i,j}$, reflecting a trade-off between model expressiveness and dimensionality reduction.

\section{Theoretical Analysis} \label{convergence_analysis}

In this section, we provide a rigorous convergence analysis for the proposed Sheaf-DMFL-Att algorithm. We establish that our method converges to a stationary point by analyzing the evolution of the feature encoders, attention weights, and the sheaf-regularized task-specific heads. We begin by stating the key assumptions underpinning our analysis.

\begin{assumption} \label{ass_W_double_stochasticity}
    Each modality-specific mixing matrix $ \bm{W}_k $ satisfies the double stochasticity condition: $\sum_{j \in \mathcal{V}_k} w_{ij}^k = 1$, $\sum_{i \in \mathcal{V}_k} w_{ij}^k = 1, \quad \forall i, j \in \mathcal{V}_k.$
\end{assumption}

\begin{assumption}\label{ass_subgraph_connectivity}
    For each modality $k \in \mathcal{M}$, the induced subgraph $\mathcal{G}_k = (\mathcal{V}_k, \mathcal{E}_k)$ is connected, where $\mathcal{V}_k = \{ i \in \mathcal{V} \mid k \in \mathcal{M}_i \}$ and $\mathcal{E}_k = \{ (i, j) \in \mathcal{E} \mid i, j \in \mathcal{V}_k \}$. This implies that the spectral gap of the corresponding mixing matrix $\bm{W}_k$ is strictly positive, i.e., $1 - |\lambda_2(\bm{W}_k)| > 0$, where $\lambda_2(\bm{W}_k)$ is the second-largest eigenvalue of $\bm{W}_k$ in magnitude.
\end{assumption}

\begin{assumption}\label{ass_smoothness}
    Each local loss function $f_i: \mathbb{R}^{d_i} \to \mathbb{R}$ is differentiable and $L$-smooth for some constant $L > 0$. That is, for any $\bm{\theta}_i, \bm{\theta}'_i \in \mathbb{R}^{d_i}$, the gradient $\nabla f_i$ is Lipschitz continuous:
    \begin{align}
         \| \nabla f_i(\bm{\theta}_i) - \nabla f_i(\bm{\theta}'_i) \| \le L \| \bm{\theta}_i - \bm{\theta}'_i \|.
    \end{align}
\end{assumption}

\begin{assumption}\label{ass_bounded_domain}
    The parameter vectors are uniformly bounded. Specifically, there exist positive constants $D_\omega$, $D_{\bm{\phi}_k}$, and $D_{\bm{\beta}_k}$ such that for all clients $i \in \mathcal{V}$ and all modalities $k \in \mathcal{M}_i$:
    \begin{align}
        \| \bm{\omega}_i \| \le D_\omega, \quad \| \bm{\phi}_{i,k} \| \le D_{\bm{\phi}_k}, \quad \text{and} \quad \| \bm{\beta}_{i,k} \| \le D_{\bm{\beta}_k}.
    \end{align}
\end{assumption}

Assumptions \ref{ass_W_double_stochasticity} and \ref{ass_subgraph_connectivity} are standard in decentralized optimization literature, ensuring that the gossip-based averaging of feature encoders is well-behaved and leads to consensus on connected subgraphs. Assumption \ref{ass_smoothness} is a standard regularity condition on the loss functions, which bounds the change in the gradient and is essential for guaranteeing descent. Finally, Assumption \ref{ass_bounded_domain} ensures that the model parameters and their projections remain in a compact set, which is necessary to bound terms related to the update of the restriction maps.
Next, we start by recalling the update of each variable of the optimization problem \ref{eq_FMTL_problem3}. At each iteration $r$, each client $i$ updates its local feature extractor for modality every modality $k \in \mathcal{M}_i$
\begin{align}
\label{eq_encoder_GD_step}
\bm{\phi}_{i,k}^{r+\frac{1}{2}} &= \bm{\phi}_{i,k}^{r} - \eta_{\phi} \nabla_{\bm{\phi}_{i,k}} f_i(\bm{\theta}_i^r),
\end{align}
\begin{align}
\label{eq_encoder_agg_step}
\bm{\phi}_{i,k}^{r+1} = \sum_{j \in \mathcal{N}_i^k} w_{ij}^k \bm{\phi}_{j,k}^{r+\frac{1}{2}},
\end{align}
with the attention, parameters are updated locally as
\begin{align}
    \label{eq_attention_GD_step}    
    \bm{\beta}_{i,k}^{r+1} &= \bm{\beta}_{i,k}^{r} - \eta_{\beta}           \nabla_{\bm{\beta}_{i,k}} f_i(\bm{\theta}_i^r).
    \end{align}
Finally, the task-specific parameters and the restriction maps are updated as in \eqref{eq:task_specific_update} and \eqref{eq:projection_update}.

Additionally, we define $\bm{\Phi}_k^r = [\bm{\phi}_{i,k}^r \text{ for } i \in \mathcal{V}_k] \in \mathbb{R}^{d_k \times |\mathcal{V}_k|}$ to be the matrix concatenating the parameters for modality $k$ from all clients $i \in \mathcal{V}_k$ at iteration $r$.
The update rules in \eqref{eq_encoder_GD_step} and \eqref{eq_encoder_agg_step} can be written in matrix form for each modality $k$ as
\begin{align}
   \bm{\Phi}_k^{r+1} = \left( \bm{\Phi}_k^r - \eta_\phi \bm{\Delta}_k^r \right) \bm{W}_k,
\end{align}
where $\bm{\Delta}_k^r = [\nabla_{\bm{\phi}_{i,k}} f_i(\bm{\theta}_i^r) \text{ for } i \in \mathcal{V}_k]$ is the concatenation of local gradients with respect to $\phi_{i,k}$ for modality $k$ and clients $i \in \mathcal{V}_k$ at iteration $r$, $\eta_\phi$ is the learning rate, and $\bm{W}_k$ satisfies Assumption \ref{ass_W_double_stochasticity} and Assumption \ref{ass_subgraph_connectivity} holds. Furthermore, we define $\bar{\bm{\phi}}_k := \frac{1}{|\mathcal{V}_k|} \sum_{i \in \mathcal{V}_k} \bm{\phi}_{i,k}$ as the average feature extractor parameters for modality $k$ over all clients $i \in \mathcal{V}_k$.

Our analysis proceeds in two main steps. First, in Lemma \ref{lemma1}, we show that the average of the feature encoders follows a standard gradient descent-like trajectory. Second, in Lemma \ref{lemma2}, we leverage this result to establish a one-step descent inequality for a global objective function evaluated at a modified parameter vector that incorporates the averaged encoders. This lemma decouples the analysis of the encoders from the other parameters and forms the foundation for our main convergence theorem, which is similar to that in \cite{tian2023distributed}.
Specifically, for client $i$, we define the modified parameter vector $\tilde{\bm{\theta}}_i^r$ as 
$\tilde{\bm{\theta}}_i^r = \left[ (\bar{\bm{\phi}}_k^r)^\top_{k \in \mathcal{M}_i}, (\bm{\beta}_{i,k}^r)^\top_{k \in \mathcal{M}_i}, (\bm{\omega}_i^r)^\top \right]^\top.$
This vector uses the averaged feature extractors $\bar{\bm{\phi}}_k^r$ for $k \in \mathcal{M}_i$, and its local parameters $\bm{\beta}_{i,k}^r$ and $\bm{\omega}_i^r$ from iteration $r$. Specifically, in our analysis, the performance is evaluated under the averaged global model parameters $\bar{\bm{\phi}}_k^r$ across clients, which are collaboratively updated using DSGD for each modality, together with the task-specific parameters, including the local attention weights $\bm{\beta}_{i,k}^r$ and classifier parameters $\bm{\omega}_i^r$. This motivates the definition of the modified parameter vector $\tilde{\bm{\theta}}_i^r$, which captures both shared and task-specific components, enabling a unified analysis of the model's behavior. Given Assumptions \ref{ass_W_double_stochasticity}-\ref{ass_bounded_domain}, we state the following lemmas.

\begin{lemma}
\label{lemma1}
For each modality $k$, the average feature extractor $\bar{\bm{\phi}}_k^r$ evolves as
\begin{align}
\bar{\bm{\phi}}_k^{r+1} - \bar{\bm{\phi}}_k^r = -\eta_{\phi} \, \bar{\bm{g}}_k^r,
\end{align}
where $\bar{\bm{g}}_k^r = \frac{1}{|\mathcal{V}_k|} \sum_{i \in \mathcal{V}_k} \nabla_{\bm{\phi}_{i,k}} f_i(\bm{\theta}_i^r)$ is the average gradient for modality $k$ at iteration $r$.
\end{lemma}
% Concatenated version of \theta
\begin{lemma}
\label{lemma2}
Under Assumptions 1-3, for the update rules~\eqref{eq_encoder_GD_step}--\eqref{eq_attention_GD_step} with step sizes satisfying $\eta_\beta < \frac{2}{L}$ and $\eta_\phi < \frac{2|\mathcal{V}_k|}{L}$ for all $k \in \mathcal{M}$, the following inequality holds
\begin{align}
&f(\tilde{\bm{\theta}}^{r+1}) 
\leq f(\tilde{\bm{\theta}}^r) - \eta_\beta \left(1 - \frac{L \eta_\beta}{2}\right) \left\| \nabla_{\bm{\beta}} f(\tilde{\bm{\theta}}^r) \right\|^2 \nonumber \\
&\quad - \sum_{k=1}^{M} 
    \eta_\phi |\mathcal{V}_k|^2 \left(1 - \frac{L \eta_\phi}{2|\mathcal{V}_k|} \right) 
    \left\| \nabla_{\bm{\phi}_k} f(\tilde{\bm{\theta}}^r) \right\|^2 \nonumber \\
&\quad +  \left[ \nabla_{\bm{\omega}} f(\tilde{\bm{\theta}}^r)^\top (\bm{\omega}^{r+1} - \bm{\omega}^r) + \frac{L}{2} \left\| \bm{\omega}^{r+1} - \bm{\omega}^r \right\|^2 \right].
\end{align}
where $\tilde{\bm{\theta}} = [\tilde{\bm{\theta}}_1^\top, \ldots, \tilde{\bm{\theta}}_N^\top]^\top$, $\bm{\beta} = \left[ \bm{\beta}_{i,k}^\top \right]_{i \in \mathcal{V},\, k \in \mathcal{M}_i}^\top$, 
$\bm{\omega} = \left[ \bm{\omega}_i^\top \right]_{i \in \mathcal{V}}^\top$, and $f(\tilde{\bm{\theta}}) := \sum_{i=1}^N f_i(\tilde{\bm{\theta}}_i)$.
\end{lemma}
\begin{proof}
Proofs of Lemma \ref{lemma1} and \ref{lemma2} can be found in Appendix \ref{appendix_proof_lemmas}. 
\end{proof}
Before stating Theorem~\ref{theorem1} to facilitate the analysis, we define the global objective function $\Psi(\tilde{\bm{\theta}}, \bm{P})$ as
\begin{align}
\label{eq_global_obj_psi_analysis}
\Psi(\tilde{\bm{\theta}}, \bm{P}) 
&:= \sum_{i=1}^N f_i(\tilde{\bm{\theta}}_i) + \frac{\lambda}{2} \bm{\omega}^\top L_\mathcal{F}(\bm{P}) \bm{\omega},
\end{align}
where the Sheaf Laplacian $L_\mathcal{F}$ is defined as in \eqref{eq_quadratic_sheaf_laplacian_omega}.
\begin{theorem}
\label{theorem1}
Under Assumptions 1-4, with step sizes $\alpha < \frac{2}{L}$, $\eta_\beta < \frac{2}{L}$, $\eta_\phi < \frac{2|\mathcal{V}_k|}{L}$ for all $k \in \mathcal{M}$, and $\eta < \frac{2}{\lambda N D_\omega^2}$, the average squared gradient norm of $\Psi$ after $R$ iterations satisfies
\begin{align}
\frac{1}{R} \sum_{r=0}^{R-1} \left\| \nabla \Psi(\tilde{\bm{\theta}}^r, \bm{P}^r) \right\|^2 \leq \frac{1}{\rho R} \left( \Psi(\tilde{\bm{\theta}}^0, \bm{P}^0) - \Psi^* \right),
\end{align}
where $\rho\!=\!\min\left\{\alpha\left(1\!-\!\frac{L\alpha}{2}\right)\!,\eta_\beta\left(1-\frac{L\eta_\beta}{2}\right)\!,\rho_\phi,\eta\left(1-\frac{\eta\lambda D_\omega^2}{2}\right)\right\}$, $\rho_\phi=\min_k\left\{\eta_\phi|\mathcal{V}_k|^2\left(1-\frac{L\eta_\phi}{2|\mathcal{V}_k|}\right)\right\}$, and $\Psi^*$ is the lower bound of $\Psi$.
\end{theorem}
\begin{proof}
The proof is deferred to Appendix \ref{appendix_proof_theorem1}.
\end{proof}

\section{Use Cases Studies}\label{Simulation_section}
In large-scale 6G communication scenarios, distributed clients collaborate to solve wireless communication tasks using their multimodal datasets. To assess the effectiveness of the proposed algorithms, we conduct a comprehensive evaluation across two specific use cases in 6G wireless communication: cooperative link blockage prediction \cite{wu2023proactively} and distributed mmWave beamforming prediction \cite{charan2022towards} as shown in Fig. \ref{fig:system_model}. 

We compare our proposed algorithms, Sheaf-DMFL and Sheaf-DMFL-Att, with three different baseline methods: decentralized stochastic gradient descent (DSGD), DMML-KL \cite{yin2024knowledge}, a stand-alone implementation where each client trains solely on local multimodal data referred to as ``Local'', and a centralized baseline trained on the full set of combined datasets to represent an upper bound on performance. Here, DSGD is implemented as a unimodal baseline where encoders for the same modality are averaged across clients, but no information is shared between different modalities. Moreover, for client groups that have access to multiple modalities, we report the best performance among unimodal models for both DSGD and Local baselines. For a fair comparison, the centralized baseline utilizes both modalities, following the same architecture described in Section \ref{proposed_algorithm}-A.

\subsection{Scenario I: Link Blockage Prediction}

\subsubsection{Scenario and Dataset Description}
In this use case, as shown in Fig. \ref{fig:system_model}(a), we consider a mmWave communication system where a BS equipped with a $U$-element phased antenna array serves a stationary user using directional beamforming. To achieve situational awareness of the surrounding environment and potential dynamic obstacles, the BS is also equipped with a LiDAR sensor and a camera. These sensors provide complementary spatial information useful for predicting link blockages.

The BS utilizes a predefined beamforming codebook $\mathcal{F} = \{ \bm{f}_q \}_{q=1}^{Q}$, where each beam $\bm{f}_q$ steers the signal towards a direction to cover a certain field of view. In the setup mentioned in \cite{wu2023proactively}, a phased array with $U = 16$ elements is used and a codebook of $Q = 64$ beams, uniformly covering the angular range $[-\pi/4, \pi/4]$.

To model wireless dynamics, we assume a block fading channel over duration $\tau_B$, with OFDM transmission over $C$ subcarriers. Let $\bm{h}_c[t] \in \mathbb{C}^U$ represent the channel vector at subcarrier $c$ and time $t \in \mathbb{Z}$. When beam $\bm{f}_q$ is applied, the received signal at subcarrier $c$ is $s_{c, q}[t] = \bm{h}_c^{H}[t] \bm{f}_q x[t] + n_c[t]$, where $s_{c, q}[t]$ is a symbol with unit average power, and $n_c[t] \sim \mathcal{CN}(0, \sigma_n^2)$ is AWGN.
At time $t$, the BS receives pilot signals from the user over all $Q$ beams. The aggregated received power vector is defined as $\bm{s}[t] = \left[ |s_1[t]|^2, \dots, |s_Q[t]|^2 \right]^\top$, where $|s_q[t]|^2 = \sum_{c=1}^C |s_{c,q}[t]|^2$ is the total power at beam $q$.
Let $T_{\text{ob}}$ be the observation window length. The sequence of past beam power vectors is $\mathcal{R}_{\text{ob}} = \left\{ \bm{r}[t+n] \right\}_{n = -T_{\text{ob}} + 1}^{0}$.

In parallel, the LiDAR sensor provides a point cloud $L[t] \in \mathbb{R}^{P \times 2}$, where each point is an (angle, distance) pair. The LiDAR observation window is $\mathcal{L}_{\text{ob}} = \left\{ L[t+n] \right\}_{n = -T_{\text{ob}} + 1}^{0}$.
The camera provides RGB images $I[t] \in \mathbb{R}^{H \times W \times C}$. The image observation window is $\mathcal{I}_{\text{ob}} = \left\{ I[t+n] \right\}_{n = -T_{\text{ob}} + 1}^{0}$.

We define the complete multimodal observation as $\mathcal{S}_{\text{ob}} = \{ \mathcal{R}_{\text{ob}}, \mathcal{L}_{\text{ob}}, \mathcal{I}_{\text{ob}} \}$.
Let $p[t] \in \{0,1\}$ denote the binary link status at time $t$, where $p[t]=1$ indicates a blocked LOS link and $p[t]=0$ otherwise. We define the future blockage indicator $b[t]$ as:
\begin{equation}
    b[t] = 
    \begin{cases}
        0, & \text{if } p[t+n] = 0, \ \forall n \in \{1, \dots, T_f\} \\
        1, & \text{otherwise}.
    \end{cases}
\end{equation}
where $T_f$ denotes the prediction horizon.

The goal is to learn a function $\chi_{\bm{\theta}}: \mathcal{S}_{\text{ob}} \rightarrow \hat{b}$, parameterized by $\bm{\theta}$, that predicts the future blockage indicator $\hat{b}$ from multimodal observation $\mathcal{S}_{\text{ob}}$, by minimizing the loss of binary classification $ f(\bm{\theta}) = \sum_{j=1}^{N_t} \ell\left( \chi_{\bm{\theta}}(\mathcal{S}_{\text{ob},j}), b^*_{j} \right)$, where $\ell(\cdot)$ denotes the binary cross-entropy loss and $N_t$ is the number of training samples.

We extend this setup to a distributed system with $N$ BSs, each with access to a local dataset consisting of a subset of modalities $\mathcal{M}_i \subseteq \{\mathcal{R}_{\text{ob}}, \mathcal{L}_{\text{ob}}, \mathcal{I}_{\text{ob}}\}$, where $|\mathcal{M}_i| \leq 3$.

\subsubsection{Experimental Settings}

We consider $N=20$ BSs, each serving a single user and having access to a local dataset consisting of a subset of modalities $\mathcal{M}_i \subseteq \{\mathcal{R}_{\text{ob}}, \mathcal{L}_{\text{ob}}, \mathcal{I}_{\text{ob}}\}$, where $|\mathcal{M}_i| \leq 3$. Each data point consists of multimodal observations collected over an observation window of $T_{\text{ob}} = 16$ time instances (i.e., 1.6 seconds), and the goal is to predict link blockage occurrence within a future prediction horizon of $T_f = 5$ instances (i.e., 1 second), following the problem formulation described in \cite{wu2023proactively}. The BSs communicate over a predefined graph to collaboratively learn models without sharing raw data.
Based on their available modalities, clients are divided into three groups: \textbf{Group 1:} LiDAR and Images, \textbf{Group 2:} LiDAR and mmWave, and \textbf{Group 3:} Images and mmWave. The client grouping and communication topology are shown in Figure~\ref{fig:client_group_blockage}.

\begin{figure}[tb]
    \centering
    \includegraphics[width=0.8\linewidth]{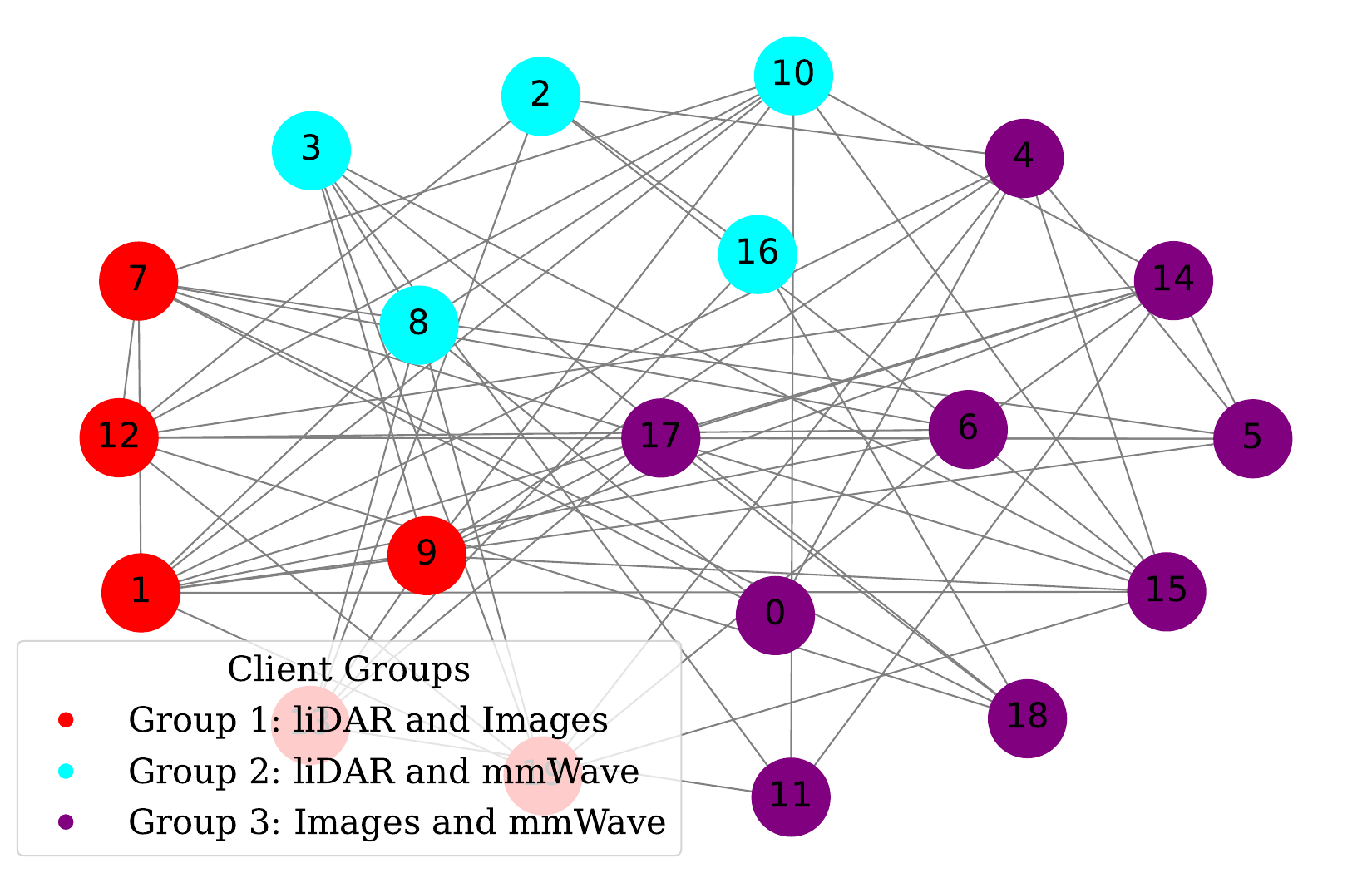}
    \caption{Client groups (BSs) with available modalities over a communication graph}
    \label{fig:client_group_blockage}
\end{figure}

Each modality is processed by a dedicated convolutional neural network (CNN) tailored to its input shape. The LiDAR feature extractor receives $16 \times 2$-channel polar representations and applies three CNN blocks, each comprising a \texttt{Conv2d-ReLU-MaxPool-BatchNorm} sequence, progressively. An adaptive average pooling layer followed by a dropout and a fully connected layer to produce a compact embedding. The RGB camera model operates on stacked image sequences of shape $16 \times 64 \times 64$, using three CNN stages with channel sizes 32, 64, and 128. The resulting $8 \times 8$ feature maps are flattened, passed through a dropout layer, and transformed via a two-layer MLP into the feature representation.

The mmWave feature extractor processes $16 \times 64$ beam energy maps. Two CNN stacks reduce width while preserving height, followed by two standard CNN blocks to capture spatiotemporal structure. For multimodal fusion, we employ an attention-based fusion classifier, where each feature extractor's backbone provides an embedding vector. Modality-specific attention weights are computed via an MLP and used to perform a weighted feature fusion. The resulting aggregated representation is passed to a final linear layer for classification.

\subsubsection{Experimental Results}

\begin{figure*}[t]
    \centering
    \begin{subfigure}{0.32\textwidth}
        \centering
        \includegraphics[width=\textwidth]{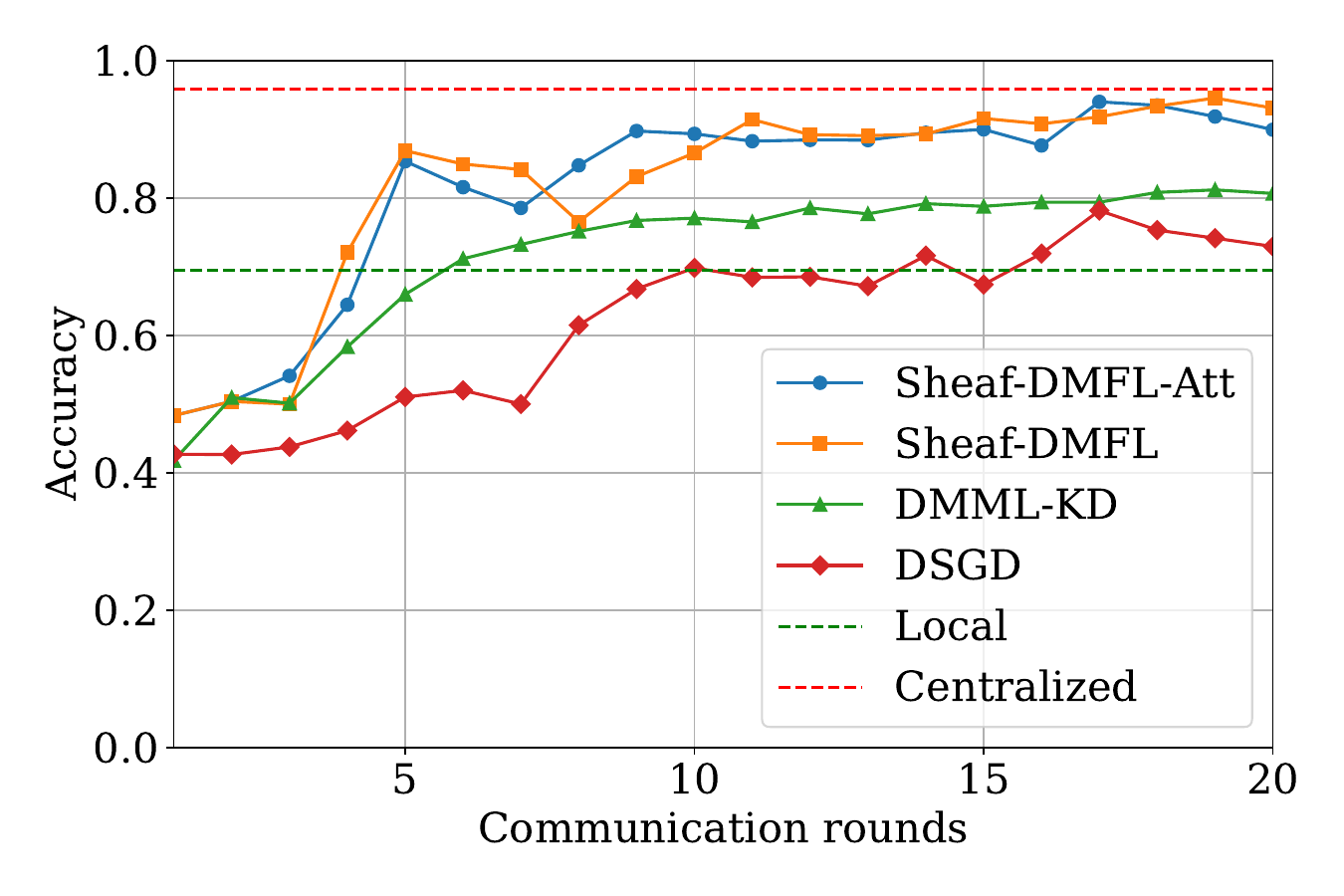}
        \caption{LiDAR + Camera}
    \end{subfigure}
    \begin{subfigure}{0.32\textwidth}
        \centering
        \includegraphics[width=\textwidth]{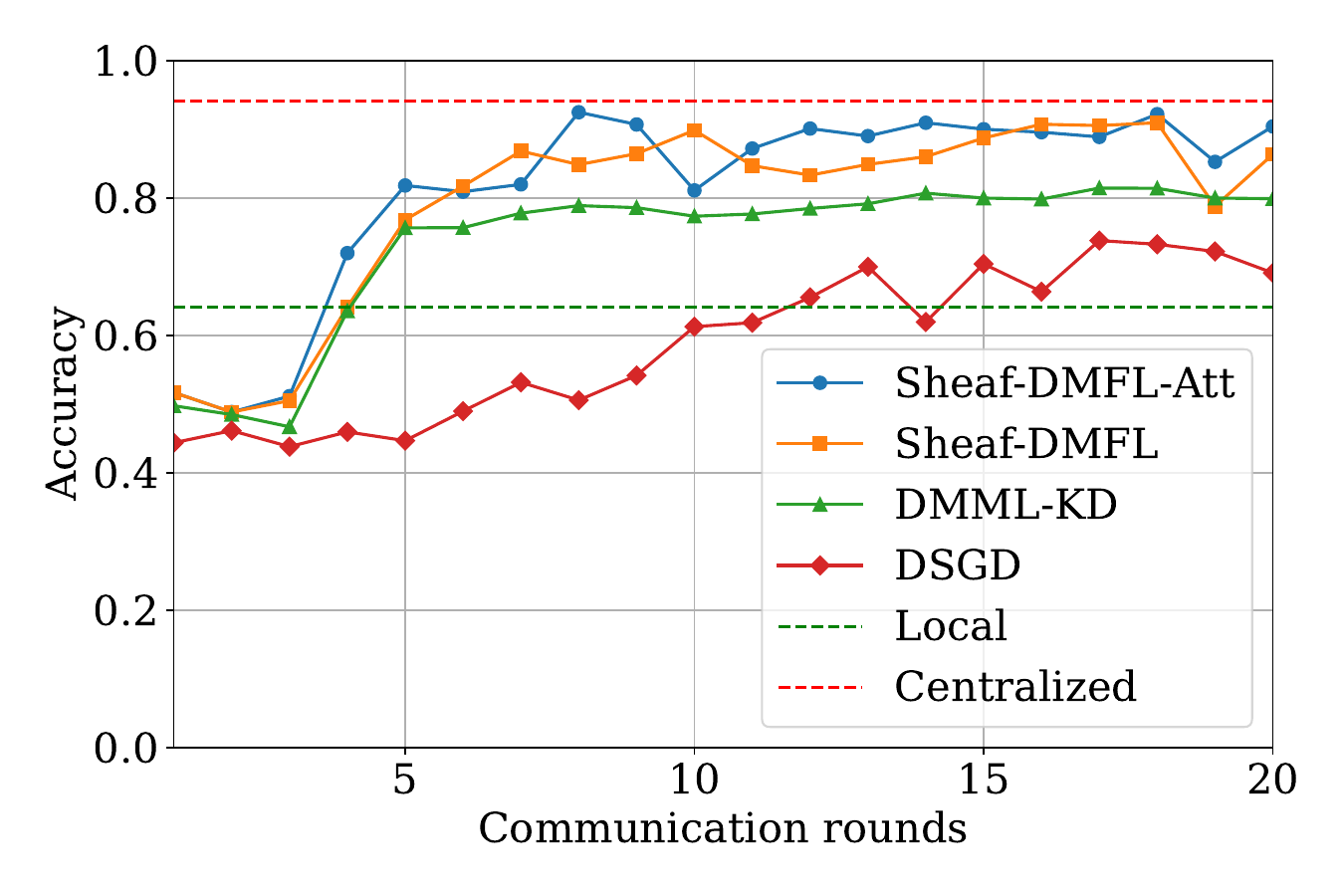}
        \caption{LiDAR + RF}
    \end{subfigure}
    \begin{subfigure}{0.32\textwidth}
        \centering
        \includegraphics[width=\textwidth]{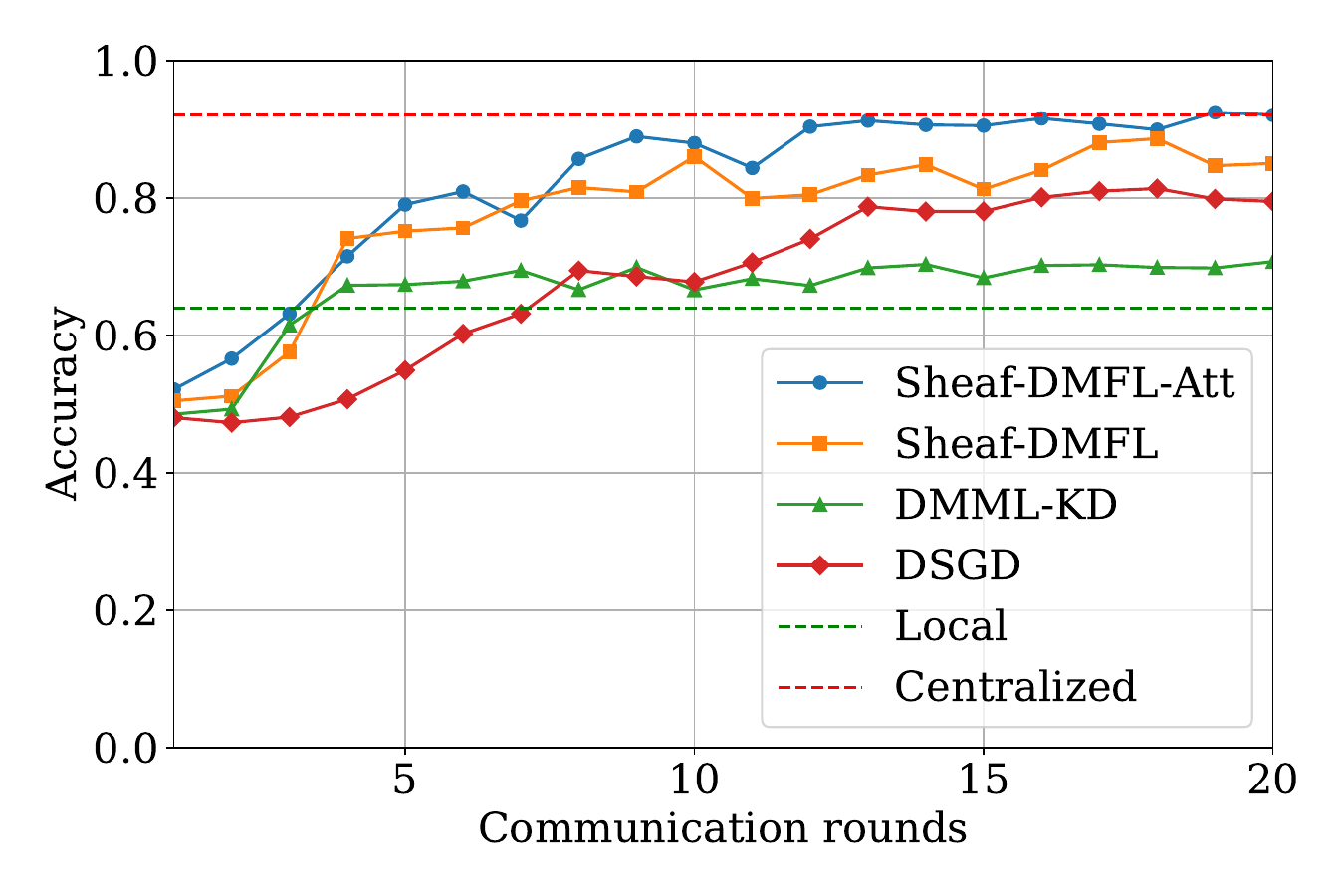}
        \caption{RF + Camera}
    \end{subfigure}

    \caption{Test accuracy as a function of the number of communication rounds for different modality combinations: (a) LiDAR + Camera, (b) LiDAR + RF, and (c) RF + Camera, using the DeepSense blockage prediction dataset.}
    \label{fig:experiment2_comparison}
\end{figure*}

\textbf{Testing accuracy. }Figure~\ref{fig:experiment2_comparison} illustrates the test accuracy for future blockage prediction within a 500\,ms prediction window using 100\,ms sampling intervals, across three multimodal combinations: (a) LiDAR + Camera, (b) LiDAR + RF, and (c) RF + Camera. These configurations represent distributed client groups with partial modality access in mmWave systems. In all cases, both Sheaf-DMFL and Sheaf-DMFL-Att consistently outperform baseline methods (Local, DSGD, and DMML-KD), showcasing their effectiveness in multimodal federated learning. Notably, Sheaf-DMFL-Att achieves the highest accuracy, highlighting the benefits of attention-driven fusion in addressing modality importance imbalance, particularly in settings involving the RF modality, which cannot capture far-ahead blockage events.

The LiDAR + Camera Group (Fig.~\ref{fig:experiment2_comparison}a) benefits from spatially rich modalities, allowing Sheaf-DMFL-Att to converge rapidly toward centralized performance. In contrast, DSGD saturates after only ten communication rounds and yields the lowest accuracy, only marginally better than the local baseline, due to its reliance on single-modality aggregation. The LiDAR + RF case (Fig.~\ref{fig:experiment2_comparison}b) further highlights the utility of Sheaf-DMFL-Att, while the RF + Camera scenario (Fig.~\ref{fig:experiment2_comparison}c) demonstrates its robustness to the absence of LiDAR through adaptive attention mechanisms. The performance gap between Sheaf-DMFL and Sheaf-DMFL-Att is most evident in the latter case, underscoring the importance of attention in noisy or ambiguous settings.

Although DMML-KD performs reasonably well in settings involving the LiDAR modality, it consistently underperforms compared to the proposed methods. Unlike the beam prediction task, which involves a complex 64-class label space, blockage prediction is a binary classification problem, indicating whether a blockage occurs or not, making it inherently simpler and contributing to the overall high accuracy (above 60\%) observed across methods. However, DMML-KD's reliance on naive averaging of feature extractors limits its ability to generalize across heterogeneous modalities, which explains its relatively lower performance in all groups, especially in the absence of important modalities such as LiDAR.

\subsection{Scenario II: Distributed MmWave Beamforming Prediction}
\label{sec:experiment1}
\subsubsection{Scenario and Dataset Description}

We additionally consider a more complicated task: distributed mmWave beamforming prediction \cite{charan2022towards}, as shown in Fig. \ref{fig:system_model}(b). The dataset comprises a co-existing camera, GPS, and mmWave beam training data. The BS is equipped with an $U$-element uniform linear array and employs a predefined beamforming codebook $F = \{ \bm{f}_q \}_{q=1}^{Q}$, where each beamforming vector $\bm{f}_q \in \mathbb{C}^{U \times 1}$ is designed to direct energy toward a specific spatial direction. 
Let $\bm{h}_c[t] \in \mathbb{C}^{U \times 1}$ denote the channel vector between the BS and the drone at subcarrier $c$ and time step $t$. The received signal at the drone can be written as 
\begin{align}
    s_c[t] = \bm{h}_c^H[t] \bm{f}[t] x[t] + n_c[t],
\end{align}
where, $x \in \mathbb{C}$ is the transmitted symbol with power constraint $\mathbb{E}[|x[t]|^2] = P$, $\bm{f}[t] \in F$ is the selected beamforming vector at time $t$, $n_c[t] \sim \mathcal{CN}(0, \sigma^2)$ is additive Gaussian noise, and $\bm{h}_c^H[t]$ is the conjugate transpose (Hermitian) of the channel vector at time $t$. The objective of the BS is to select the optimal beamforming vector $\bm{f}^*[t]$ that maximizes the SNR at the receiver. The optimal beamforming vector is given by 
\begin{align}
   \bm{f}^*[t] = \arg\max_{\bm{f}_q \in F} \frac{1}{K} \sum_{k=1}^{K} \left| \bm{h}_k^H[t] \bm{f}_q \right|^2 
\end{align}

Acquiring explicit CSI in mmWave/THz systems is challenging, making conventional exhaustive beam searches costly. Instead, learning-based approaches predict the optimal beam index using sensory data. Following \cite{charan2022towards}, we consider a vision-based BS system that integrates GPS and real-time drone images. Since GPS alone provides location coordinates, a \textit{GPS modality} fuses latitude, longitude, distance, and height, enabling efficient multimodal data representation.

Formally, the set of available sensory information at time $t$ is defined as $\bm{x}[t] = \{ g[t], I[t] \}$, where $g[t] \in \mathbb{R}^4$ represents the drone's augmented GPS measurements (latitude, longitude, distance, height), $I[t] \in \mathbb{R}^{W \times H \times C}$ is the RGB image captured by the BS. The goal is to learn a mapping function $\chi_{\bm{\theta}}: \bm{x}[t] \rightarrow \hat{f}[t]$, parameterized by $\bm{\theta}$, that predicts the optimal beam index $\hat{f}[t]$ based on sensory inputs by minimizing the objective $f(\bm{\theta}) = \sum_{j=1}^{N_t} \ell\left( \chi_{\bm{\theta}}(\bm{x}_{j}), f^*_{j} \right)$ where $\ell(\cdot)$ is the categorical cross-entropy loss function, and $N_t$ is the size of the dataset.  

\subsubsection{Experimental Settings}
We consider a distributed system of $N = 20$ BSs, each serving a distinct drone and having access only to its local dataset with modalities $\mathcal{M}_i \subset \mathcal{M}$ where $|\mathcal{M}_i| \leq 2$ (GPS and/or images). BSs cooperate via a predefined communication graph that connects clients sharing at least one modality and are grouped by available modalities: Group 1 (GPS only), Group 2 (images only), and Group 3 (both), as illustrated in Fig.~\ref{fig:client_group}. Using the DeepSense Drone dataset~\cite{charan2022towards}, each BS predicts the optimal mmWave beam for its drone. To handle heterogeneous modalities, we use group-specific models: a ResNet-50-based feature extractor (pre-trained on ImageNet) for images and a two-layer MLP for GPS data. Following~\cite{charan2022towards}, the ResNet-50 output is mapped to a beam index of a codebook with size $Q=64$. For multimodal clients, we employ a multi-stream architecture, fusing and passing features through a shared logistic regression layer.

\begin{figure}[b]
    \centering
    \includegraphics[width=0.8\linewidth, trim=10 10 12 12, clip]{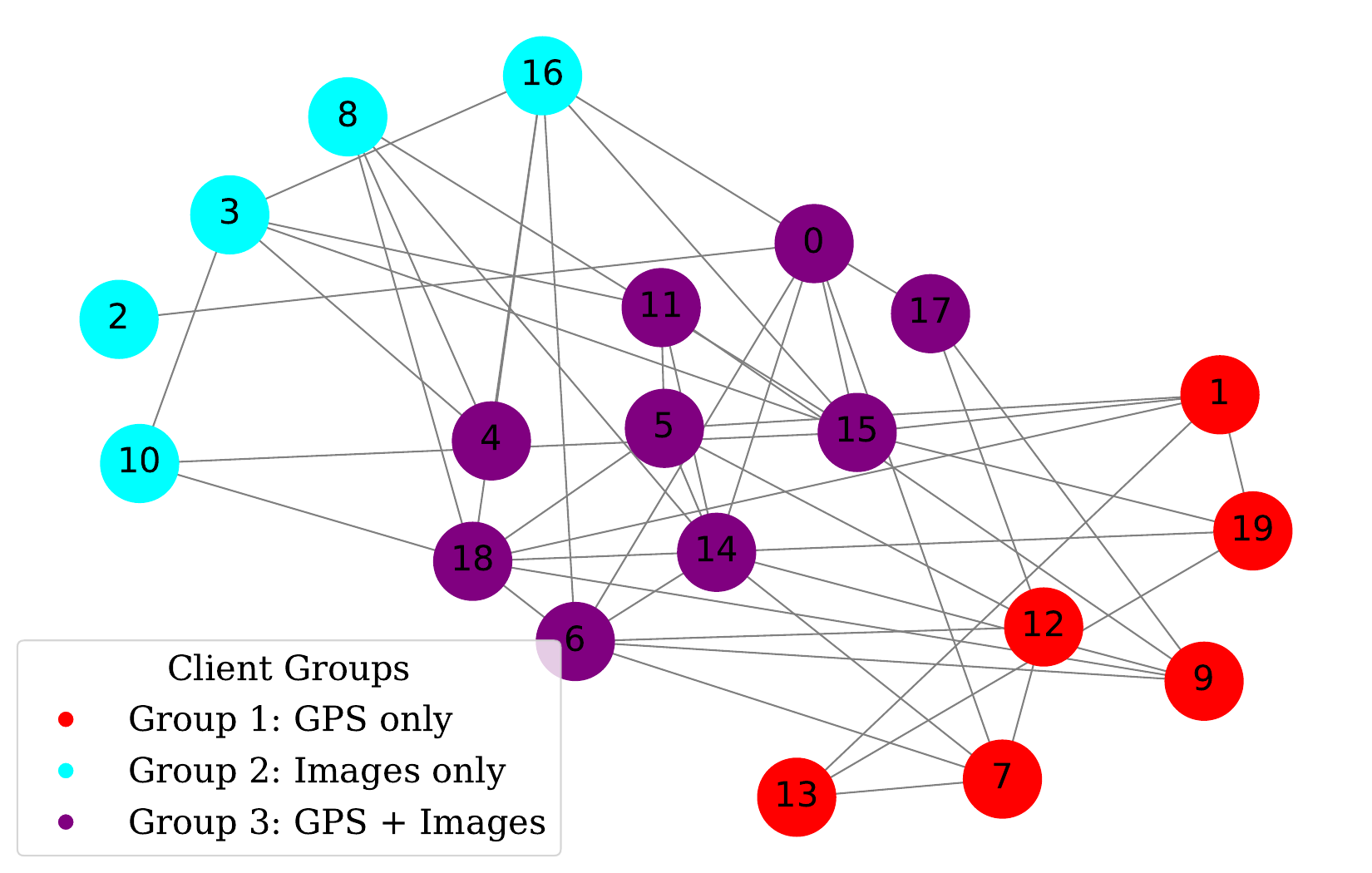}
    \caption{Client groups (BSs) with available modalities over a communication graph.}
    \label{fig:client_group}
\end{figure}

\subsubsection{Experimental Results}
\begin{figure*}[t]
    \centering
    \begin{subfigure}{0.32\textwidth}
        \centering
        \includegraphics[width=\textwidth]{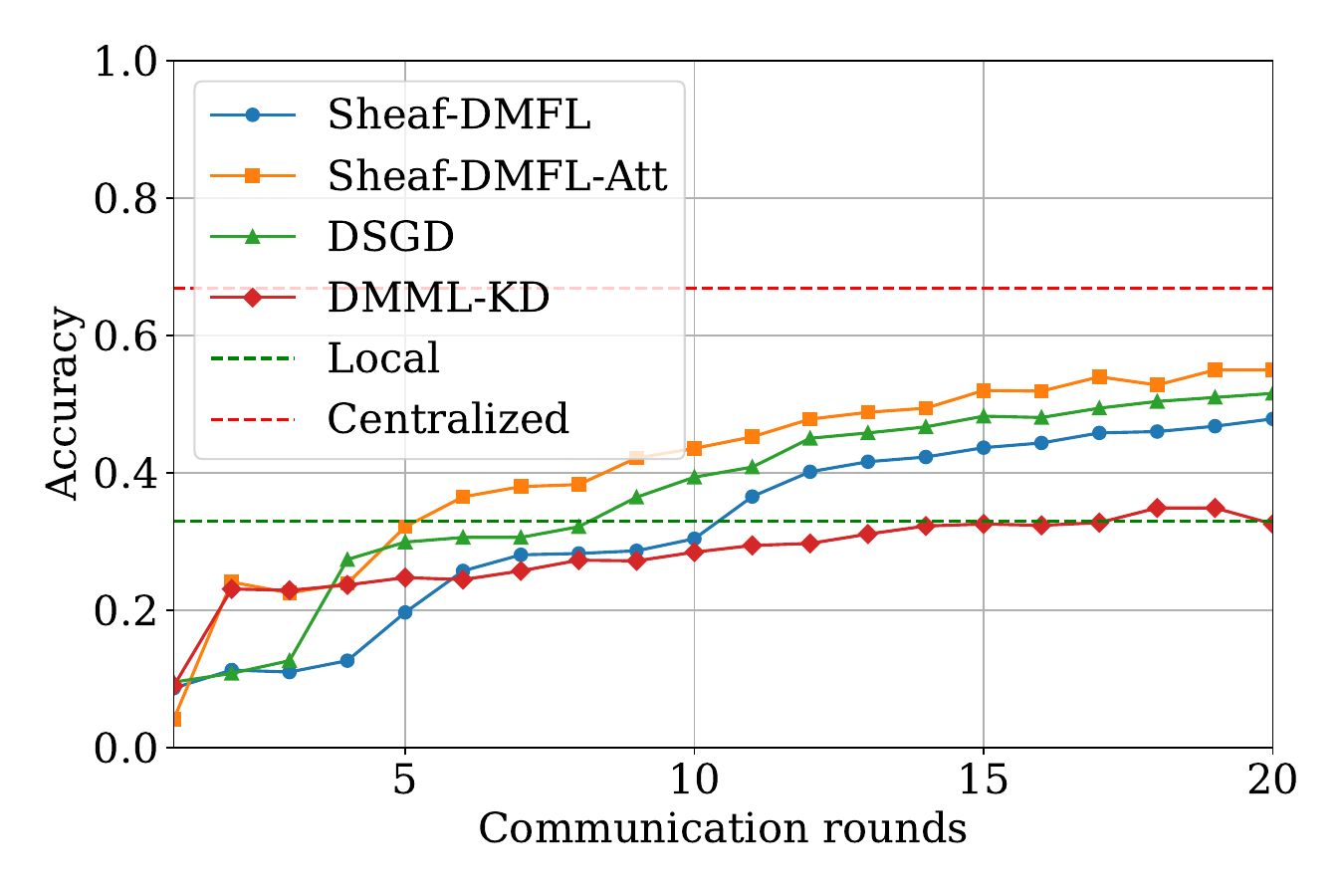}
        \caption{GPS only}
        \label{subfig_drone_group1}
    \end{subfigure}
    \begin{subfigure}{0.32\textwidth}
        \centering
        \includegraphics[width=\textwidth]{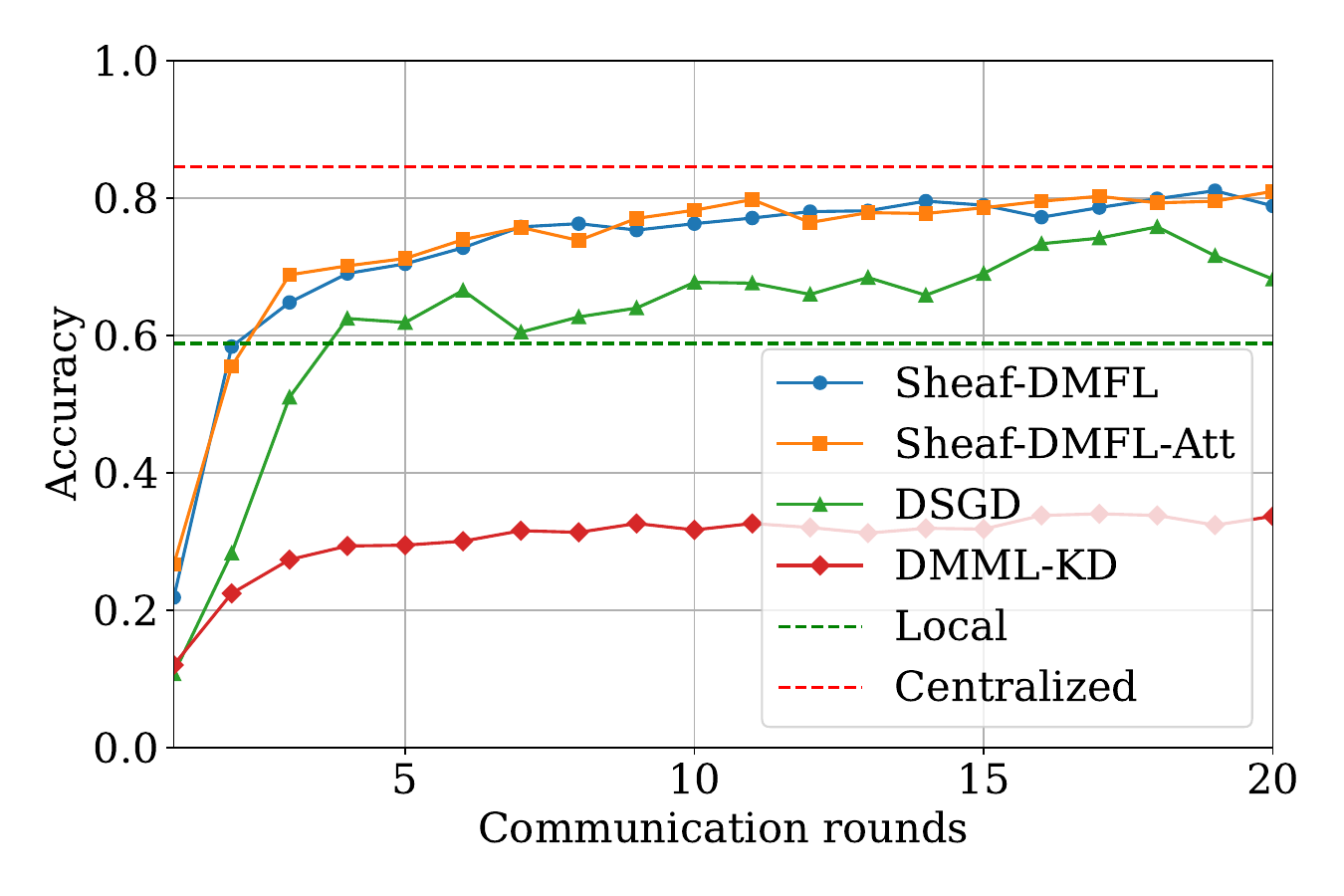}
        \caption{Camera only}
        \label{subfig_drone_group2}
    \end{subfigure}
    \begin{subfigure}{0.32\textwidth}
        \centering
        \includegraphics[width=\textwidth]{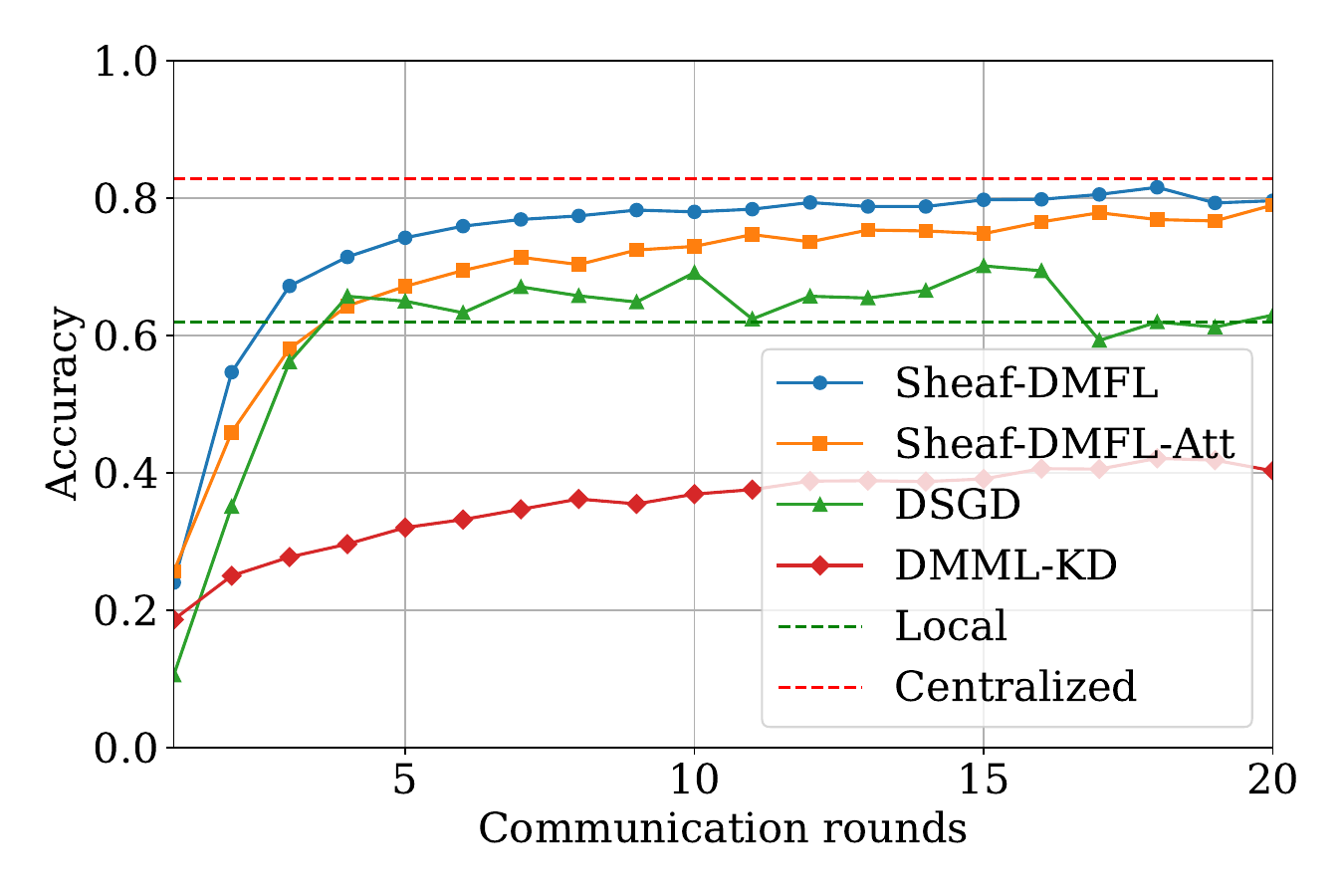}
        \caption{GPS + Camera}
        \label{subfig_drone_group3}
    \end{subfigure}

    \caption{Test accuracy as a function of the number of communication rounds for different modality combinations: (a) GPS only, (b) Camera only, and (c) GPS + Camera, using the DeepSense Drone dataset.}
    \label{fig:experiment1_comparison}
\end{figure*}

\textbf{Testing accuracy. }At the beginning of this experiment, we evaluate the performance of the proposed Sheaf-DMFL and Sheaf-DMFL-Att algorithms under an IID data split across clients, and compare them with the other baselines. 
%The results demonstrate that our proposed Sheaf-DMFL approach consistently outperforms the decentralized baselines, DSGD and Local, across different modality combinations. 
The results are shown in Figure~\ref{fig:experiment1_comparison}, where Sheaf-DMFL achieves both faster convergence and higher accuracy than DSGD and Local, particularly for the clients with \textit{Images-only} and \textit{GPS+Images} modalities, as depicted in Figures \ref{subfig_drone_group2} and \ref{subfig_drone_group3}, respectively. Notably, Sheaf-DMFL reaches high accuracy within the initial communication rounds and maintains stable performance throughout training. This highlights its ability to effectively leverage multimodal data and coordinate learning across heterogeneous clients by modeling task relationships through the sheaf structure. 

Notably, in Figure~\ref{fig:experiment1_comparison}(a), Sheaf-DMFL performs poorly for the GPS-only group compared with Sheaf-DMFL-Att. This result highlights a key limitation of simple feature concatenation in Sheaf-DMFL, motivating the design of Sheaf-DMFL-Att. 
Specifically, in the standard Sheaf-DMFL, the task-specific head $\bm{\omega}_i$ of a GPS-only client~$i$ is regularized towards the head $\bm{\omega}_j$ of its multimodal neighbors via the sheaf term. 
If a neighbor~$j$ has a powerful image modality, its loss gradient $\nabla_{\bm{\omega}_j}$ will be predominantly shaped by visual features. 
This strong, image-driven signal then pulls $\bm{\omega}_i$ towards a configuration that is ill-suited for the GPS-only task, a phenomenon known as \textit{negative transfer}. 
On the other hand, Sheaf-DMFL-Att directly mitigates this issue by incorporating an attention mechanism \textit{before} the task-specific head. This allows each multimodal client to learn a locally optimal, weighted fusion of its features, where the subsequent gradient is more balanced and represents a more holistic understanding of the task. This ensures that the knowledge transferred through the sheaf structure is beneficial to all neighbors, including unimodal ones. Consequently, Sheaf-DMFL-Att improves the performance of the GPS-only group while maintaining its state-of-the-art accuracy in other settings, underscoring the critical role of adaptive feature fusion in heterogeneous decentralized learning.

Finally, DMML-KD underperforms compared to the other baselines due to its dependence on model aggregation, which is not well-suited for heterogeneous modality settings or when one modality is dominant. The aggregation mechanism is unable to effectively align the diverse encoder representations learned by clients with different available modalities, resulting in inconsistent feature spaces and degraded global knowledge distillation.

\begin{table}[t]
    \centering
    \caption{Test accuracy of Sheaf-DMFL compared to baselines for DeepSense Drone dataset with partial or occluded view.}
    \setlength{\tabcolsep}{3pt}
    \begin{tabular}{lccc}
        \hline
        \textbf{Algorithm} & \textbf{GPS Only} & \textbf{Images Only} & \textbf{GPS + Images} \\
        \hline
        Centralized & 0.58 & 0.62 & 0.65 \\
        Sheaf-DMFL & 0.49 & 0.61 & 0.64 \\
        Sheaf-DMFL-Att & 0.52 & 0.59 & 0.64 \\
        DMML-KD & 0.36 & 0.26 & 0.37 \\
        DSGD & 0.49 & 0.47 & 0.53 \\
        Local & 0.33 & 0.56 & 0.56 \\
        \hline
    \end{tabular}
    \label{tab:accuracy_results}
\end{table}

\textbf{Robustness. }Secondly, we assess the robustness of each method in scenarios where the image modality provides only a partial or occluded view of the scene.  In this setting, the dataset is modified by partially occluding the image input to simulate real-world scenarios, where the drone may move out of the camera’s field of view. As a result, the GPS modality becomes more critical for accurate prediction. Under these perturbed conditions, where the image modality is compromised, the performance increasingly depends on the complementary \textit{GPS} input. The results are presented in Table~\ref{tab:accuracy_results}, which provide deeper insights into the performance of each algorithm under limited-view conditions. In this scenario, Sheaf-DMFL outperforms the Local baseline and remains competitive with DSGD on the \textit{GPS-only} task, demonstrating its flexibility and robustness in single-modality learning. More importantly, Sheaf-DMFL continues to lead in the multimodal setting, effectively fusing partial visual information with GPS data, maintaining strong predictive performance. These results emphasize the value of multimodal learning, particularly in environments where one modality may be unreliable. Sheaf-DMFL excels by adaptively integrating \textit{GPS} signals, achieving higher accuracy than single-modality baselines. Furthermore, Sheaf-DMFL-Att outperforms all other baselines, particularly in the \textit{GPS-only} group, due to its ability to leverage both modalities more effectively and dynamically attend to their relative importance.

\textbf{Impact of heterogeneity. }Additionally, we investigate the impact of heterogeneity in visual perspectives, arising from variations in drone viewpoints across different clients. In such settings, collaborative training without personalization may degrade performance due to conflicting or domain-specific visual features. To illustrate this, we manually perturb the visual inputs for each client to simulate diverse observational angles and environments, as shown in Figure~\ref{fig:user_different_views}. This highlights the need for personalized models that adapt to client-specific data distributions while leveraging collaborative training.

\begin{figure}[tb]
    \centering
    \begin{subfigure}{0.32\columnwidth}
        \centering
        \includegraphics[width=\linewidth]{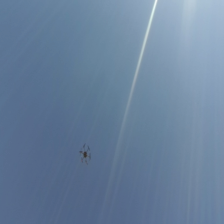}
    \end{subfigure}
    \begin{subfigure}{0.32\columnwidth}
        \centering
        \includegraphics[width=\linewidth]{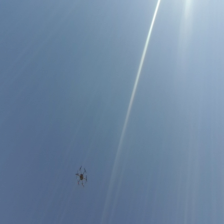}
    \end{subfigure}
    \begin{subfigure}{0.32\columnwidth}
        \centering
        \includegraphics[width=\linewidth]{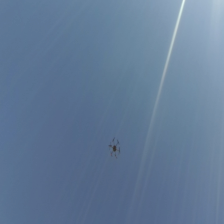}
    \end{subfigure}
    
    \caption{Camera views of clients 2 (a), 3 (b), and 6 (c) for the same drone location.}
    \label{fig:user_different_views}
\end{figure}

\begin{figure}[tb]
        \centering
        \includegraphics[width=0.5\textwidth]{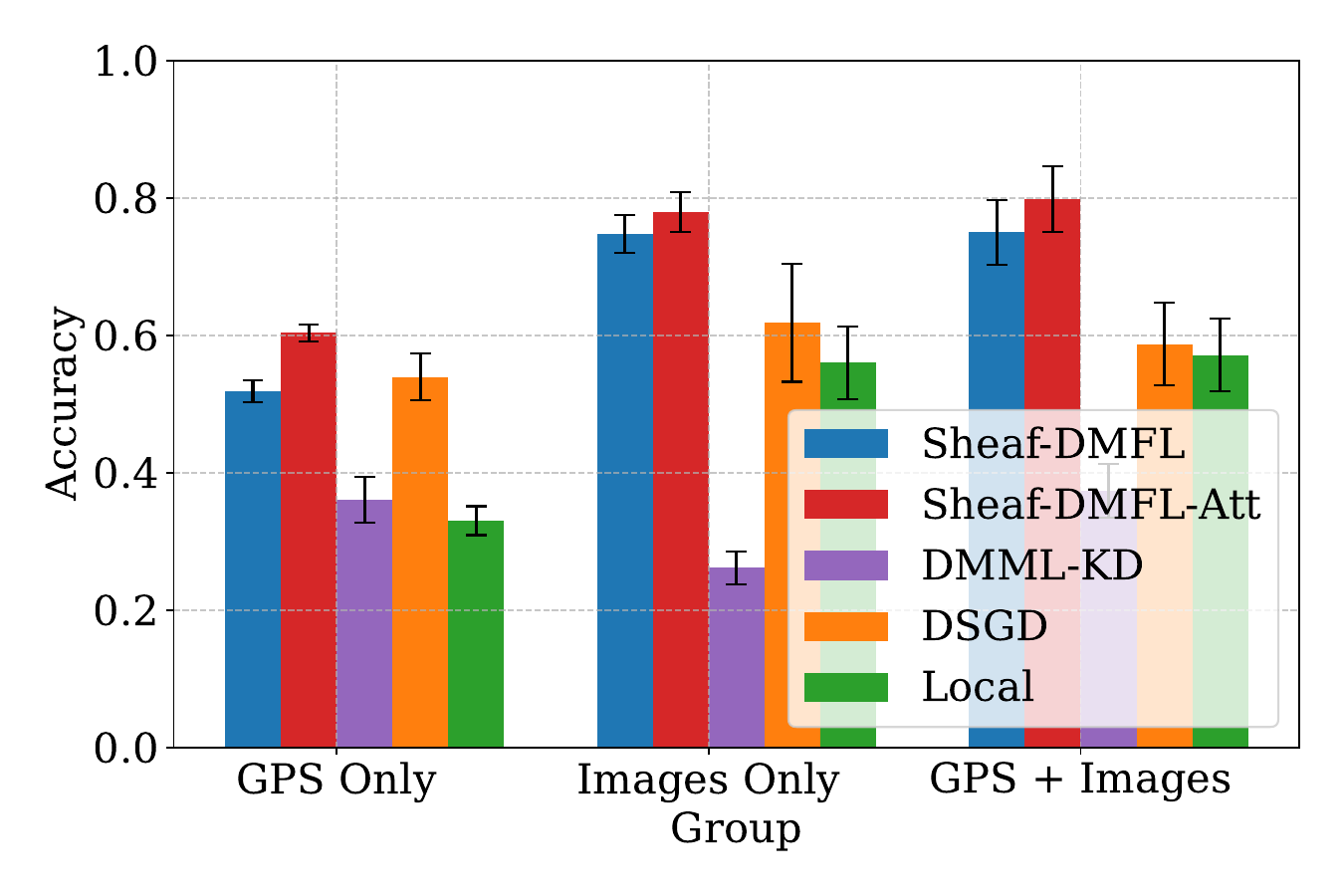}
    \caption{Comparison of test accuracy for different modality combinations using different views.}
    \label{fig:accuracy_different_view}
\end{figure}

Figure~\ref{fig:accuracy_different_view} highlights the effectiveness of our proposed approach in the challenging setting of different views (heterogeneity), where each drone observes distinct visual scenes due to variations in camera perspective or coverage. This scenario reflects realistic conditions in which drones operate in different environments, leading to heterogeneous visual inputs across clients. In this setting, Sheaf-DMFL-Att consistently outperforms decentralized baselines by effectively learning personalized models while leveraging shared information. Unlike DSGD, whose performance degrades due to global averaging across inconsistent visual features, Sheaf-DMFL and Sheaf-DMFL-Att preserve client-specific representations, enabling more accurate predictions despite visual misalignment. Notably, Sheaf-DMFL-Att achieves the highest performance across all modality groups in the different-view setting. 

These results underscore the strength of our approach in federated multimodal learning under realistic, heterogeneous conditions, where clients have different views and modalities. In contrast, baseline methods based solely on averaging, such as DSGD and DMML-KD, struggle to maintain performance due to their inability to personalize effectively or dynamically integrate modalities.

\begin{table}[t]
    \centering
    \caption{Test accuracy for different groups using different $\gamma$ and initialization of restriction maps $\bm{P}_{i,j}.$}
    \label{tab:ablation_accuracy}
    \begin{tabular}{lcc}
        \toprule
        \textbf{Group} & $\bm{P}_{i, j}^{0}$ & \textbf{Accuracy} \\
        \midrule
        \multirow{4}{*}{GPS only} 
        & $\mathbf{P}_{i,j}^{\mathcal{I}}$ ($\gamma$=0.1)  & $0.5833 \pm 0.0471$ \\
        & $\mathbf{P}_{i,j}^{\mathcal{I}}$ ($\gamma$=0.23) & $0.6052 \pm 0.0345$ \\
        & $\mathbf{P}_{i,j}^{\mathcal{R}}$ ($\gamma$=0.1) & $0.5644 \pm 0.0477$ \\
        & $\mathbf{P}_{i,j}^{\mathcal{R}}$ ($\gamma$=0.23) & $0.5694 \pm 0.0451$ \\
        \midrule
        \multirow{4}{*}{Camera only} 
        & $\mathbf{P}_{i,j}^{\mathcal{I}}$ ($\gamma$=0.1)  & $0.7917 \pm 0.0282$ \\
        & $\mathbf{P}_{i,j}^{\mathcal{I}}$ ($\gamma$=0.23) & $0.8179 \pm 0.0310$ \\
        & $\mathbf{P}_{i,j}^{\mathcal{R}}$ ($\gamma$=0.1) & $0.8083 \pm 0.0471$ \\
        & $\mathbf{P}_{i,j}^{\mathcal{R}}$ ($\gamma$=0.23) & $0.8154 \pm 0.0305$ \\
        \midrule
        \multirow{4}{*}{Camera + GPS} 
        & $\mathbf{P}_{i,j}^{\mathcal{I}}$ ($\gamma$=0.1)  & $0.8102 \pm 0.0418$ \\
        & $\mathbf{P}_{i,j}^{\mathcal{I}}$ ($\gamma$=0.23) & $0.8201 \pm 0.0474$ \\
        & $\mathbf{P}_{i,j}^{\mathcal{R}}$ ($\gamma$=0.1) & $0.7910 \pm 0.0384$ \\
        & $\mathbf{P}_{i,j}^{\mathcal{R}}$ ($\gamma$=0.23) & $0.8115 \pm 0.0358$ \\
        \bottomrule
    \end{tabular}
\end{table}
\textbf{Ablation study.} Finally, we conduct an ablation study for Sheaf-DMFL-Att to gauge the effect of the factor $\gamma$ and the projection matrix initialization $\bm{P}_{i, j}^0$ on the performance. The results, shown in Table~\ref{tab:ablation_accuracy}, report test accuracy for $\gamma \in \{0.1, 0.23\}$ and $\bm{P}_{i, j}^0 \in \{\bm{P}_{i, j}^\mathcal{I}, \bm{P}_{i, j}^\mathcal{R}\}$. In general, a larger $\gamma$ results in a more expressive projection of the weights, leading to improved accuracy across all modality groups. This trend is evident, as accuracy values are consistently higher for $\gamma = 0.23$ compared to $\gamma = 0.1$. However, the choice of initialization also plays a crucial role. The $\mathbf{P}_{i,j}^{\mathcal{I}}$ initialization, based on structured identity-guided projection, consistently yields better or comparable performance relative to the random initialization $\mathbf{P}_{i,j}^{\mathcal{R}}$, suggesting that structured weight projection facilitates more stable and effective cross-client adaptation. However, the projection matrix $\bm{P}_{i,j}$ has a dimension of $\lfloor \gamma \cdot \frac{d_i + d_j}{2} \rfloor \times d_i$, which grows substantially with larger $\gamma$. This increased dimensionality offers higher expressiveness but comes at the cost of greater memory and storage requirements, highlighting a trade-off between performance and efficiency in deployment.

\section{Conclusion}
\label{sec_conclusion}

This work proposed a novel decentralized multimodal learning framework that leverages the sheaf structure and local attention mechanism to enable collaborative learning among clients with heterogeneous data modalities. By formulating the decentralized multimodal learning problem as multi-task learning, the proposed algorithms capture and exploit task relationships across clients through learnable sheaf restriction maps. This formalism allows each client to benefit from neighboring devices with similar tasks while preserving privacy. We provide a rigorous theoretical analysis for Sheaf-DMFL-Att under standard non-convex settings, establishing its convergence to a stationary point. The effectiveness of the proposed algorithm is validated through experiments on practical scenarios, including blockage prediction and mmWave beamforming. 
In both tasks, the proposed algorithms outperform the decentralized baselines by converging faster and achieving higher accuracy, where Sheaf-DMFL-Att shows superiority over Sheaf-DMFL for the more challenging mmWave beamforming scenario. Future work will investigate more efficient methods for modeling the sheaf structure to improve scalability by minimizing the storage of large restriction maps per edge by means of compression and pruning.

\appendix 

\subsection{Proofs of Lemmas}
\label{appendix_proof_lemmas}

\subsubsection{Proof of Lemma 1}
The average parameter vector for modality $k$ at iteration $r$ is $\bar{\bm{\phi}}_k^r = \frac{1}{|\mathcal{V}_k|} \bm{\Phi}_k^r \mathbf{1}_{|\mathcal{V}_k|}$.
Given that assumptions~\ref{ass_W_double_stochasticity} and \ref{ass_subgraph_connectivity} hold, we analyze the difference between consecutive average parameters for modality $k$:
\begin{align}
\label{eq_avg_phi_update_lemma}
\bar{\bm{\phi}}_k^{r+1} = \frac{1}{|\mathcal{V}_k|} \bm{\Phi}_k^{r+1} \mathbf{1}_{|\mathcal{V}_k|} 
%\nonumber &= \frac{1}{|\mathcal{V}_k|} \left( \bm{\Phi}_k^r - \eta_\phi \bm{\Delta}_k^r \right) \bm{W}_k \mathbf{1}_{|\mathcal{V}_k|} \\
%\nonumber &= \frac{1}{|\mathcal{V}_k|} \left( \bm{\Phi}_k^r - \eta_\phi \bm{\Delta}_k^r \right) \mathbf{1}_{|\mathcal{V}_k|} \\
%\nonumber &= \frac{1}{|\mathcal{V}_k|} \bm{\Phi}_k^r \mathbf{1}_{|\mathcal{V}_k|} - \frac{\eta_\phi}{|\mathcal{V}_k|} \bm{\Delta}_k^r \mathbf{1}_{|\mathcal{V}_k|} \\
= \bar{\bm{\phi}}_k^r - \eta_\phi \bar{\bm{g}}_k^r,
\end{align}
where $\bar{\bm{g}}_k^r = \frac{1}{|\mathcal{V}_k|} \sum_{i \in \mathcal{V}_k} \nabla_{\bm{\phi}_{i,k}} f_i(\bm{\theta}_i^r)$ is the average gradient for modality $k$, at iteration $r$, which proves Lemma \ref{lemma1}.
\subsubsection{Proof of Lemma 2}

Following Assumption~\ref{ass_smoothness}, we have for each client $i$
\begin{align}
\label{eq_smoothnes_expansion_revised}
& f_i\left( \tilde{\bm{\theta}}_i^{r+1} \right) - f_i\left( \tilde{\bm{\theta}}_i^{r} \right) 
\leq  \nonumber \\
& \nabla f_i\left( \tilde{\bm{\theta}}_i^{r} \right)^\top \left( \tilde{\bm{\theta}}_i^{r+1} - \tilde{\bm{\theta}}_i^{r} \right) + \frac{L}{2} \left\| \tilde{\bm{\theta}}_i^{r+1} - \tilde{\bm{\theta}}_i^{r} \right\|^2 \nonumber \\
&\stackrel{(a)}{=} \sum_{k \in \mathcal{M}_i} \left(\nabla_{\bar{\bm{\phi}}_k} f_i(\bar{\bm{\phi}}_k^r)\right)^\top (\bar{\bm{\phi}}_k^{r+1} - \bar{\bm{\phi}}_k^r) + \frac{L}{2} \sum_{k \in \mathcal{M}_i} \|\bar{\bm{\phi}}_k^{r+1} - \bar{\bm{\phi}}_k^r\|^2 \nonumber \\
&\quad + \sum_{k \in \mathcal{M}_i} \left(\nabla_{\bm{\beta}_{i,k}} f_i(\bm{\beta}_{i,k}^r)\right)^\top (\bm{\beta}_{i,k}^{r+1} - \bm{\beta}_{i,k}^r) \nonumber \\ 
&\quad + \frac{L}{2} \sum_{k \in \mathcal{M}_i} \|\bm{\beta}_{i,k}^{r+1} - \bm{\beta}_{i,k}^r\|^2 + \left(\nabla_{\bm{\omega}_i} f_i(\bm{\omega}_i^r)\right)^\top (\bm{\omega}_i^{r+1} - \bm{\omega}_i^r) \nonumber \\
&\quad + \frac{L}{2} \|\bm{\omega}_i^{r+1} - \bm{\omega}_i^r\|^2.
\end{align}
where $(a)$ holds from the definition of $\tilde{\bm{\theta}}_i^r$. 
By expanding the smoothness inequality in \eqref{eq_smoothnes_expansion_revised}, and utilizing the updates of $\bm{\beta}_{i,k}$ and $\bar{\bm{\phi}}_k$ in Lemma \ref{lemma1}, we can derive a descent bound on the overall objective as
\begin{align} 
\label{eq_attention_smoothness_revised_4} 
f(\tilde{\bm{\theta}}^{r+1})  &\leq f(\tilde{\bm{\theta}}^r) -\eta_\beta \left(1 - \frac{L \eta_\beta}{2}\right) \left\| \nabla_{\bm{\beta}} f(\tilde{\bm{\theta}}^r) \right\|^2 \\ 
&\quad - \sum_{k=1}^{M} 
    \eta_\phi |\mathcal{V}_k|^2 \left(1 - \frac{L \eta_\phi}{2|\mathcal{V}_k|} \right) 
    \left\| \nabla_{\bm{\phi}_k} f(\tilde{\bm{\theta}}^r) \right\|^2 \nonumber \\ 
&\quad +  \left[ \nabla_{\bm{\omega}} f(\tilde{\bm{\theta}}^r)^\top (\bm{\omega}^{r+1} - \bm{\omega}^r) + \frac{L}{2} \left\| \bm{\omega}^{r+1} - \bm{\omega}^r \right\|^2 \right],\nonumber 
\end{align}
where $\bm{\beta} = \left[ \bm{\beta}_{i,k}^\top \right]_{i \in \mathcal{V},\, k \in \mathcal{M}_i}^\top$,
$\bm{\omega} = \left[ \bm{\omega}_i^\top \right]_{i \in \mathcal{N}}^\top$, 
where we define the concatenated modified vector $\tilde{\bm{\theta}} = [\tilde{\bm{\theta}}_1^\top, \ldots, \tilde{\bm{\theta}}_N^\top]^\top$ and $f(\tilde{\bm{\theta}}) := \sum_{i=1}^N f_i(\tilde{\bm{\theta}}_i)$, which concludes the proof for Lemma \ref{lemma2}. The detailed derivation is omitted due to space limitations.
\subsection{Proof of Theorem 1}
\label{appendix_proof_theorem1}
We write the vectorized form of the update equations in \eqref{eq:task_specific_update} and \eqref{eq:projection_update} as follows
\begin{align}
\bm{\omega}^{r+1} &= \bm{\omega}^{r} - \alpha \left( \nabla f(\bm{\omega}^{r}) + \lambda (\bm{P}^{r})^\top \bm{P}^{r} \bm{\omega}^{r} \right), \\
\bm{P}^{r+1} &= \bm{H} \odot \left( \bm{P}^{r} - \eta \lambda \bm{P}^{r} \bm{\omega}^{r+1} (\bm{\omega}^{r+1})^\top \right),
\end{align}
where $\odot$ denotes the element-wise (Hadamard) product, and $\bm{H}$ is a block matrix with blocks $\bm{H}_{ij} = \mathbf{1}_{d_{ij} \times d_i}$ whose entries are all ones. The matrix $\bm{H}$ is then defined block-wise as
\begin{align}
    \bm{H}_{ij} = 
    \begin{cases}
    \mathbf{1}_{d_{ij} \times d_i}, & \text{if } (i,j) \in \mathcal{E}, \\
    \bm{0}_{d_{ij} \times d_i}, & \text{otherwise},
    \end{cases}
\end{align}

where $\mathbf{1}_{d_{ij} \times d_i}$ denotes the all-ones matrix of size $d_{ij} \times d_i$, and $\bm{0}_{d_{ij} \times d_i}$ is the zero matrix of the same size.

Adding $\frac{\lambda}{2} (\bm{\omega}^{r+1})^\top (\bm{P}^r)^\top \bm{P}^r \bm{\omega}^{r+1}$ on both sides of \eqref{eq_attention_smoothness_revised_4}, and using the definition of $\Psi(\tilde{\bm{\theta}}^{r+1}, \bm{P}^r)$ in \eqref{eq_global_obj_psi_analysis}, we have
\begin{align}
    \label{eq_attention_smoothness_revised_5}
     &\Psi(\tilde{\bm{\theta}}^{r+1}, \bm{P}^r) \leq \nonumber \\
     & f(\tilde{\bm{\theta}}^r) - \eta_\beta \left(1 - \frac{L \eta_\beta}{2} \right) \left\| \nabla_{\bm{\beta}} \Psi(\tilde{\bm{\theta}}^{r}, \bm{P}^r) \right\|^2 \nonumber \\
     & - \sum_{k=1}^{M} 
    \eta_\phi |\mathcal{V}_k|^2 \left(1 - \frac{L \eta_\phi}{2|\mathcal{V}_k|} \right) 
    \left\| \nabla_{\bar{\bm{\phi}}_k} \Psi(\tilde{\bm{\theta}}^{r}, \bm{P}^r) \right\|^2 \nonumber \\
    & + \nabla_{\bm{\omega}} f(\tilde{\bm{\theta}}^r)^\top (\bm{\omega}^{r+1} - \bm{\omega}^r) + \frac{L}{2} \left\| \bm{\omega}^{r+1} - \bm{\omega}^r \right\|^2 \nonumber \\
    & + \frac{\lambda}{2} (\bm{\omega}^{r+1})^\top (\bm{P}^r)^\top \bm{P}^r \bm{\omega}^{r+1} \nonumber \\
    % &= f(\tilde{\bm{\theta}}^r) - \eta_\beta \left(1 - \frac{L \eta_\beta}{2} \right) \left\| \nabla_{\bm{\beta}} \Psi(\tilde{\bm{\theta}}^{r}, \bm{P}^r) \right\|^2 \nonumber \\
    %&\quad - \sum_{k=1}^{M} \eta_\phi |\mathcal{V}_k|^2 \left(1 - \frac{L \eta_\phi}{2|\mathcal{V}_k|} \right) \left\| \nabla_{\bar{\bm{\phi}}_k} \Psi(\tilde{\bm{\theta}}^{r}, \bm{P}^r) \right\|^2 \nonumber \\
    % &\quad + \nabla_{\bm{\omega}} f(\tilde{\bm{\theta}}^r)^\top (\bm{\omega}^{r+1} - \bm{\omega}^r) + \frac{L}{2} \left\| \bm{\omega}^{r+1} - \bm{\omega}^r \right\|^2 \nonumber \\
    % &\quad + \frac{\lambda}{2} (\bm{\omega}^{r+1} - \bm{\omega}^{r})^\top (\bm{P}^r)^\top \bm{P}^r \bm{\omega}^{r+1} + \frac{\lambda}{2} (\bm{\omega}^r)^\top (\bm{P}^r)^\top \bm{P}^r \bm{\omega}^r \nonumber \\
    &= \Psi(\tilde{\bm{\theta}}^r, \bm{P}^r) + \left\langle \nabla_{\bm{\omega}} \Psi(\tilde{\bm{\theta}}^r, \bm{P}^r), \bm{\omega}^{r+1} - \bm{\omega}^{r} \right\rangle \nonumber \\
    &\quad + \frac{L}{2} \left\| \bm{\omega}^{r+1} - \bm{\omega}^r \right\|^2 - \eta_\beta \left(1 - \frac{L \eta_\beta}{2} \right) \left\| \nabla_{\bm{\beta}} \Psi(\tilde{\bm{\theta}}^{r}, \bm{P}^r) \right\|^2 \nonumber \\
    &\quad - \sum_{k=1}^{M} 
    \eta_\phi |\mathcal{V}_k|^2 \left(1 - \frac{L \eta_\phi}{2|\mathcal{V}_k|} \right) 
    \left\| \nabla_{\bar{\bm{\phi}}_k} \Psi(\tilde{\bm{\theta}}^{r}, \bm{P}^r) \right\|^2,
\end{align}
where the last equality holds from the definition of the global objective function $\Psi(\tilde{\bm{\theta}}^{r+1}, \bm{P}^r)$ and the definition of the dot product $<\bm{a}, \bm{b}> = \bm{a}^\top \bm{b}$.

Now, using the update rule for $\bm{\omega}^{r+1}$, we have
\begin{align}
\label{eq_expand_omega_1}
    & \left\langle \nabla_{\bm{\omega}} \Psi(\tilde{\bm{\theta}}^r, \bm{P}^r), \bm{\omega}^{r+1} - \bm{\omega}^r \right\rangle \nonumber \\
    &= \left\langle \nabla_{\bm{\omega}} \Psi(\tilde{\bm{\theta}}^r, \bm{P}^r), -\alpha \nabla_{\bm{\omega}} \Psi(\tilde{\bm{\theta}}^r, \bm{P}^r) \right\rangle \nonumber \\
    &= -\alpha \left\| \nabla_{\bm{\omega}} \Psi(\tilde{\bm{\theta}}^r, \bm{P}^r) \right\|^2.
\end{align}
Additionally, we have
\begin{align}
\label{eq_expand_omega_2}
    \left\| \bm{\omega}^{r+1} - \bm{\omega}^r \right\|^2
    &= \alpha^2 \left\| \nabla_{\bm{\omega}} \Psi(\tilde{\bm{\theta}}^r, \bm{P}^r) \right\|^2.
\end{align}

Substitute \eqref{eq_expand_omega_1} and \eqref{eq_expand_omega_2} into \eqref{eq_attention_smoothness_revised_5}, we get
\begin{align}
\label{eq_attention_smoothness_revised_6}
    & \Psi(\tilde{\bm{\theta}}^{r+1}, \bm{P}^r) 
    \leq
    %& \Psi(\tilde{\bm{\theta}}^r, \bm{P}^r) - \alpha \left\| \nabla_{\bm{\omega}} \Psi(\tilde{\bm{\theta}}^r, \bm{P}^r) \right\|^2 + \frac{L \alpha^2}{2} \left\| \nabla_{\bm{\omega}} \Psi(\tilde{\bm{\theta}}^r, \bm{P}^r) \right\|^2 \nonumber \\
    % &\quad - \eta_\beta \left(1 - \frac{L \eta_\beta}{2} \right) \left\| \nabla_{\bm{\beta}} \Psi(\tilde{\bm{\theta}}^r, \bm{P}^r) \right\|^2 \nonumber \\
    % &\quad - \sum_{k=1}^{M} \eta_\phi |\mathcal{V}_k|^2 \left(1 - \frac{L \eta_\phi}{2|\mathcal{V}_k|} \right) \left\| \nabla_{\bar{\bm{\phi}}_k} \Psi(\tilde{\bm{\theta}}^{r}, \bm{P}^r) \right\|^2 \nonumber \\
    \Psi(\tilde{\bm{\theta}}^r, \bm{P}^r) 
    - \left( \alpha - \frac{L \alpha^2}{2} \right) \left\| \nabla_{\bm{\omega}} \Psi(\tilde{\bm{\theta}}^r, \bm{P}^r) \right\|^2 \nonumber \\
    &\quad - \eta_\beta \left(1 - \frac{L \eta_\beta}{2} \right) \left\| \nabla_{\bm{\beta}} \Psi(\tilde{\bm{\theta}}^r, \bm{P}^r) \right\|^2 \nonumber \\
    &\quad - \sum_{k=1}^{M} 
    \eta_\phi |\mathcal{V}_k|^2 \left(1 - \frac{L \eta_\phi}{2|\mathcal{V}_k|} \right) 
    \left\| \nabla_{\bar{\bm{\phi}}_k} \Psi(\tilde{\bm{\theta}}^{r}, \bm{P}^r) \right\|^2.
\end{align}
To ensure that all the norm terms have negative coefficients, the step sizes must satisfy $\alpha < \frac{2}{L}$, $\eta_\beta < \frac{2}{L}$, and $\eta_\phi < \frac{2 |\mathcal{V}_k|}{L}, \forall k \in \mathcal{M}$.

From the definition of $\Psi(\bm{\theta}, \bm{P})$, we have
\begin{align}
\label{eq_global_Psi_P}
    \Psi(\tilde{\bm{\theta}}^{r+1}, \bm{P}^{r+1}) 
    &= f(\tilde{\bm{\theta}}^{r+1}) + \frac{\lambda}{2} (\bm{\omega}^{r+1})^\top (\bm{P}^{r+1})^\top \bm{P}^{r+1} \bm{\omega}^{r+1}.
\end{align}

Given the update rule for $\bm{P}^{r+1}$, we obtain
\begin{align}
\label{eq_P_update_inequality}
    &(\bm{P}^{r+1})^\top \bm{P}^{r+1}
    \preceq  \\
    & \left( \bm{P}^r - \eta \lambda \bm{P}^r \bm{\omega}^{r+1} (\bm{\omega}^{r+1})^\top \right)^\top 
    \left( \bm{P}^r - \eta \lambda \bm{P}^r \bm{\omega}^{r+1} (\bm{\omega}^{r+1})^\top \right). \nonumber
\end{align}

Expanding the right-hand side of \eqref{eq_P_update_inequality} we get
\begin{align}
\label{eq_P_update_expand_inequality}
    &(\bm{P}^r - \eta \lambda \bm{P}^r \bm{\omega}^{r+1} (\bm{\omega}^{r+1})^\top )^\top 
    (\bm{P}^r - \eta \lambda \bm{P}^r \bm{\omega}^{r+1} (\bm{\omega}^{r+1})^\top ) \nonumber \\
    & = (\bm{P}^r)^\top \bm{P}^r - 2\eta \lambda (\bm{P}^r)^\top \bm{P}^r \bm{\omega}^{r+1} (\bm{\omega}^{r+1})^\top \nonumber \\
    &\quad + \eta^2 \lambda^2 (\bm{P}^r)^\top \bm{P}^r \bm{\omega}^{r+1} (\bm{\omega}^{r+1})^\top \bm{\omega}^{r+1} (\bm{\omega}^{r+1})^\top.
\end{align}

Substituting \eqref{eq_P_update_expand_inequality} into \eqref{eq_global_Psi_P}, we get
\begin{align}
    \label{eq_restriction_bound}
    &\Psi(\tilde{\bm{\theta}}^{r+1}, \bm{P}^{r+1}) \leq \nonumber \\
    &f(\tilde{\bm{\theta}}^{r+1}) + \frac{\lambda}{2} (\bm{\omega}^{r+1})^\top (\bm{P}^r)^\top \bm{P}^r \bm{\omega}^{r+1} \nonumber \\
    &- \eta \lambda^2 (\bm{\omega}^{r+1})^\top (\bm{P}^r)^\top \bm{P}^r \bm{\omega}^{r+1} (\bm{\omega}^{r+1})^\top \bm{\omega}^{r+1} \nonumber \\
    &+ \frac{\eta^2 \lambda^3}{2} (\bm{\omega}^{r+1})^\top (\bm{P}^r)^\top \bm{P}^r \bm{\omega}^{r+1} (\bm{\omega}^{r+1})^\top \bm{\omega}^{r+1} (\bm{\omega}^{r+1})^\top \bm{\omega}^{r+1}.
\end{align}

Furthermore, using the gradient of $\Psi$ with respect to $\bm{P}$, we can compute the squared Frobenius norm
\begin{align}
    \label{eq_restriction_gradient}
    & \left\| \nabla_{\bm{P}} \Psi(\tilde{\bm{\theta}}^{r+1}, \bm{P}^r) \right\|_F^2 \nonumber \\
    & = \lambda^2 \operatorname{Tr} \left( (\bm{\omega}^{r+1}) (\bm{\omega}^{r+1})^\top (\bm{P}^r)^\top \bm{P}^r \bm{\omega}^{r+1} (\bm{\omega}^{r+1})^\top \right) \nonumber \\
    &= \lambda^2 (\bm{\omega}^{r+1})^\top (\bm{P}^r)^\top \bm{P}^r \bm{\omega}^{r+1} \left\| \bm{\omega}^{r+1} \right\|^2.
\end{align}

Using the gradient expression in \eqref{eq_restriction_gradient}, we can rewrite \eqref{eq_restriction_bound} as
\begin{align}
\label{eq_restriction_bound_2}
    & \Psi(\tilde{\bm{\theta}}^{r+1}, \bm{P}^{r+1}) \nonumber \\
    &  \leq\Psi(\tilde{\bm{\theta}}^{r+1}, \bm{P}^{r})  - \eta \left\| \nabla_{\bm{P}} \Psi(\tilde{\bm{\theta}}^{r+1}, \bm{P}^r) \right\|_F^2 \nonumber \\
    &\quad + \frac{\eta^2\lambda}{2} \left\| \nabla_{\bm{P}} \Psi(\tilde{\bm{\theta}}^{r+1}, \bm{P}^r) \right\|_F^2 \left\| \bm{\omega}^{r+1} \right\|^2 \nonumber \\
    & = \Psi(\tilde{\bm{\theta}}^{r+1}, \bm{P}^{r}) 
    - \eta \left(1 - \frac{\eta \lambda}{2} \left\| \bm{\omega}^{r+1} \right\|^2 \right) 
    \left\| \nabla_{\bm{P}} \Psi(\tilde{\bm{\theta}}^{r+1}, \bm{P}^{r}) \right\|_F^2 \nonumber \\
    & \leq \Psi(\tilde{\bm{\theta}}^{r+1}, \bm{P}^{r}) 
    - \eta \left(1 - \frac{\eta \lambda D_\omega^2}{2} \right) 
    \left\| \nabla_{\bm{P}} \Psi(\tilde{\bm{\theta}}^{r+1}, \bm{P}^{r}) \right\|_F^2 ,
\end{align}
where the last inequality follows Assumption~\ref{ass_bounded_domain} for $\|\bm{\omega}\|$. Therefore, to ensure that the term $\left(1 - \frac{\eta \lambda D_\omega^2}{2} \right)$ is positive, we choose $\eta$ such that $\eta < \frac{2}{\lambda D_\omega^2}$. By adding both \eqref{eq_attention_smoothness_revised_6} and \eqref{eq_restriction_bound_2}, we have
\begin{align}
\label{eq_combined_inequalty}
&\Psi(\tilde{\bm{\theta}}^{r+1}, \bm{P}^{r+1}) \leq \nonumber \\ 
&\Psi(\tilde{\bm{\theta}}^r, \bm{P}^r) 
- \left( \alpha - \frac{L \alpha^2}{2} \right) \left\| \nabla_{\bm{\omega}} \Psi(\tilde{\bm{\theta}}^r, \bm{P}^r) \right\|^2 \nonumber \\
& - \eta_\beta \left(1 - \frac{L \eta_\beta}{2} \right) \left\| \nabla_{\bm{\beta}} \Psi(\tilde{\bm{\theta}}^r, \bm{P}^r) \right\|^2 \nonumber \\
& - \sum_{k=1}^{M} 
    \eta_\phi |\mathcal{V}_k|^2 \left(1 - \frac{L \eta_\phi}{2|\mathcal{V}_k|} \right) 
    \left\| \nabla_{\bar{\bm{\phi}}_k} \Psi(\tilde{\bm{\theta}}^{r}, \bm{P}^r) \right\|^2 \nonumber \\
& - \eta \left(1 - \frac{\eta \lambda N D_\omega^2}{2} \right) 
\left\| \nabla_{\bm{P}} \Psi(\tilde{\bm{\theta}}^{r+1}, \bm{P}^r) \right\|_F^2.
\end{align}

Taking the sum of \eqref{eq_combined_inequalty} from $r = 0$ to $R - 1$, we obtain the following telescoping result
\begin{align}
\label{eq_sum_combined_inequalty}
& \Psi(\tilde{\bm{\theta}}^{R}, \bm{P}^{R}) \leq \nonumber \\
&\Psi(\tilde{\bm{\theta}}^{0}, \bm{P}^{0}) 
-  \alpha \left(  1 - \frac{L \alpha}{2} \right) \sum_{r=0}^{R-1} \left\| \nabla_{\bm{\omega}} \Psi(\tilde{\bm{\theta}}^{r}, \bm{P}^{r}) \right\|^2 \nonumber \\
& -  \eta_\beta \left(1 - \frac{L \eta_\beta}{2} \right) \sum_{r=0}^{R-1}\left\| \nabla_{\bm{\beta}} \Psi(\tilde{\bm{\theta}}^{r}, \bm{P}^{r}) \right\|^2 \nonumber \\
& -  \sum_{k=1}^{M} 
    \eta_\phi |\mathcal{V}_k|^2 \left(1 - \frac{L \eta_\phi}{2|\mathcal{V}_k|} \right) 
    \sum_{r=0}^{R-1} \left\| \nabla_{\bar{\bm{\phi}}_k} \Psi(\tilde{\bm{\theta}}^{r}, \bm{P}^r) \right\|^2 \nonumber \\
& -  \eta \left(1 - \frac{\eta \lambda N D_\omega^2}{2} \right) 
\sum_{r=0}^{R-1} \left\| \nabla_{\bm{P}} \Psi(\tilde{\bm{\theta}}^{r+1}, \bm{P}^{r}) \right\|_F^2.
\end{align}

Rearranging the terms and defining $\rho = \min \left\{\alpha( 1 - \frac{L \alpha}{2}), \eta_\beta \left(1 - \frac{L \eta_\beta}{2} \right), \rho_\phi, \eta \left(1 - \frac{\eta \lambda D_\omega^2}{2} \right) \right\}$, where $\rho_\phi := \min_{k} \left\{ \eta_\phi \, |\mathcal{V}_k|^2 \left(1 - \frac{L \eta_\phi}{2 |\mathcal{V}_k|} \right) \right\}$, we have
\begin{align}
& \frac{1}{R} \sum_{r=0}^{R-1} \left\| \nabla \Psi(\tilde{\bm{\theta}}^r, \bm{P}^r) \right\|^2 
 \leq \nonumber \\
 & \frac{1}{\rho R} \left( \Psi(\tilde{\bm{\theta}}^0, \bm{P}^0) - \Psi(\tilde{\bm{\theta}}^R, \bm{P}^R) \right) \leq \frac{1}{\rho R} \left( \Psi(\tilde{\bm{\theta}}^0, \bm{P}^0) - \Psi^\star \right),
\end{align}
where $\Psi^\star$ is the lower bound of the function $\Psi$ and $\left\| \nabla \Psi(\tilde{\bm{\theta}}^k, \bm{P}^k) \right\|^2$ is defined as 
\begin{align}
    \left\| \nabla \Psi(\tilde{\bm{\theta}}^k, \bm{P}^k) \right\|^2
&= \left\| \nabla_{\bm{\beta}} \Psi(\tilde{\bm{\theta}}^k, \bm{P}^k) \right\|^2
+ \sum_{k=1}^{M}\left\| \nabla_{\bar{\bm{\phi}}_k} \Psi(\tilde{\bm{\theta}}^{r}, \bm{P}^r) \right\|^2 \nonumber \\
&\quad + \left\| \nabla_{\bm{\omega}} \Psi(\tilde{\bm{\theta}}^k, \bm{P}^k) \right\|^2 + \left\| \nabla_{\bm{P}} \Psi(\tilde{\bm{\theta}}^k, \bm{P}^k) \right\|_F^2,\nonumber
\end{align}
which concludes the proof of Theorem \ref{theorem1}.

\end{document}